\documentclass[runningheads]{llncs}
\usepackage[T1]{fontenc}
\usepackage{graphicx}
\usepackage{hologo} 
\usepackage{amssymb}
\usepackage{amsmath}
\usepackage{url}
\usepackage{hyperref}

\usepackage{enumitem}

\usepackage{lineno}

\usepackage[appendix=strip,bibliography=common]{apxproof}

\usepackage{thmtools, thm-restate}

\usepackage{booktabs} 
\usepackage{makecell} 
\usepackage{multirow} 

\usepackage{caption}
\usepackage{subcaption} 
\usepackage{pgfplots}
\pgfplotsset{compat=1.8}
\usepackage{tikz}
\usetikzlibrary{
  arrows.meta,                        
  backgrounds,                        
  ext.paths.ortho,                    
  ext.positioning-plus,               
  ext.node-families.shapes.geometric, 
  calc,                               
  bending,                            
  positioning,                        
  plotmarks,
  bending,
  quotes,
  patterns,
  shadings,
  decorations.text,
  shapes.symbols}
\tikzstyle{link} = []
\tikzstyle{arrow} = [-{Stealth[length=5pt,round,inset=3pt,width=7pt,flex'=1]},line width=0.5pt, 
quotes/.style={fill=white, inner sep=2pt, rounded corners}]
\tikzstyle{node} = [circle,draw=black, minimum size=7mm, inner sep=0pt]

\usepackage{graphicx}
\usepackage[table]{xcolor}

\def\r{1.3} 
\def\rg{1.5} 
\def\rp{1.6} 

\usepackage[]{todonotes}

\newcommand{\com}[1]{}

\newcommand{\candidates}{\mathcal{C}}
\newcommand{\n}{n}
\newcommand{\m}{m}

\newcommand{\NL}[2]{\text{NL}_{#1}(#2)}

\newcommand{\mssm}{\mbox{MSs\hspace{-\the\fontcharwd\font`s+0.2ex}\textsuperscript{\raisebox{0.4ex}{$-$}}}}  

\newcommand{\mssp}{\mbox{MSs\hspace{-\the\fontcharwd\font`s+0.2ex}\textsuperscript{\raisebox{0.4ex}{$+$}}}}

\newcommand{\X}{\mathcal{X}}
\newcommand{\Y}{\mathcal{Y}}

\DeclareMathOperator{\borda}{BO}
\DeclareMathOperator{\cop}{CO}
\DeclareMathOperator{\tc}{TC}
\DeclareMathOperator{\uc}{UC}
\DeclareMathOperator{\wuc}{wUC}
\DeclareMathOperator{\mm}{MM}

\DeclareMathOperator{\argmin}{argmin}
\DeclareMathOperator{\argmax}{argmax}

\makeatletter
\newcommand{\blfootnote}[1]{%
  \begingroup
  \let\old@makefntext\@makefntext
  \long\def\@makefntext##1{\noindent ##1}%
  \footnotetext{#1}
  \endgroup
}
\makeatother

\begin{document}
%
\title{Characterizing Necessary Losers to Explain Tournaments Solutions}
%
%
\author{Clément Contet\inst{1,2} \and Umberto Grandi\inst{1,3} \and Jérôme Mengin\inst{1,2}}
%
\institute{Institut de Recherche en Informatique de Toulouse (IRIT) \and
Université de Toulouse \and
Université Toulouse Capitole}
\maketitle              
\begin{abstract}

We study the problem of formally explaining why a candidate was not selected by a given tournament rule, by identifying sub-tournaments in which the candidate loses independently of how the rest of the tournament is completed. We define destructive minimal supports as any minimal sub-tournament satisfying this property, which in formal explainable artificial intelligence corresponds to abductive explanations for the question “Why does the loser lose the tournament?”. 
For six common tournament solutions (maximin, uncovered set and its weighted variant, top cycle, Copeland, and Borda) we provide characterizations of when a candidate is either a necessary loser or a possible winner, we determine the size of the smallest destructive minimal supports, complemented by polynomial-time algorithms for their computation except for the case of Borda and Copeland rules which we conjecture to also be polynomial.

\end{abstract}

\section{Introduction}
\blfootnote{This paper is the extended version of Contet, Grandi, Mengin. Characterizing Necessary Losers to Explain Tournaments Solutions. In: Proceedings of the 9th International Conference on Algorithmic Decision Theory (ADT) (2026).} Why do people accept the outcome of decision making processes?
If it seems natural that people follow decisions which are favorable to them, it is less clear to see why they would conform to disadvantageous outcomes.
Procedural justice theory, developed in the late 1990s and early 2000s by Tyler~\cite{tyler1994governing,tyler2000social} and Hibbing and Theiss-Morse~\cite{hibbing2001process}, claims that the willingness to accept and to obey the outcome of a decision making process not only depends on the outcome per se but also on the process that leads to it. If the primary influence on the individual position toward an outcome is its favorability, a process perceived with a higher procedural fairness, i.e., which treats all stakeholders fairly, increases the legitimacy and trustworthiness of the decision, the process in itself and more generally the responsible entity. This in turn leads to more acceptance and conformance with the decision.
More recently, Carman~\cite{carman2010process} showed that this theory holds for participatory democracy, in particular for petitioning systems.
It is important to realize that it is all about \emph{perceived} fairness, as Tyler states: ``When authorities are presenting their decisions to the people influenced by them, they need to make clear that they have listened to and considered the arguments made. They can do so by accounting for their decisions.''~\cite{tyler2000social}.
A practical example of this is the work of Suryanarayana et al.~\cite{SuryanarayanaEtAlAAMAS2022} establishing that providing explanations for election outcomes could increase satisfaction and acceptance of voters, especially for those with less favorable outcomes.
These results as well as numerous voices in the computational social choice community~\cite{GrossiDigitalDem2024} have called for efforts to make collective decision-making processes more accessible to a wider audience by developing tools to explain their inner workings. 
More generally, this movement is in line with the field of \emph{Explainable Artificial Intelligence (XAI)} which aims at making algorithms more understandable to human users (see \cite{dwivedi2023explainable,miller2023explainable} for an introduction).

Building on previous work on transparency \cite{de2014does}, de Fine Licht and de Fine Licht~\cite{de2020artificial} and Grimmelikhuijsen~\cite{grimmelikhuijsen2023explaining} showed that from a social sciences perspective, the basic transparency approach consisting of simply making public all the information to enable anyone to recompute and verify algorithmic decisions is insufficient to foster trust.
\footnote{This body of works distinguishes between \emph{fishbowl transparency}, i.e., full openness of the decision-making process, and \emph{transparency in rationale}, that is public receiving dedicated information for the justification or explanation of the decision. Our notion of \emph{explainability} is associated to the latter.}
For collective decisions, this implies that simply publishing the individual vote counts and the rule used may not be enough.

In the computational social choice literature, multiple methods to specifically explain or justify the outcome of an election have been proposed.
In a classical social choice approach, a first line of works unfolds some logical reasonings based on desirable axiomatic properties to justify the election of one alternative~\cite{boixel2022calculus,CaillouxEndrissAAMAS2016,NardiEtAlAAMAS2022,peters2020explainable}.
A second method presents the stakeholders with manually or automatically curated statistics on the expressed preferences supporting a specific outcome~\cite{SuryanarayanaEtAlAAMAS2022}.
Finally, a third line of work~\cite{contet2024abductive,contet2026explaining} uses abductive reasoning to explain why an alternative is in the winning set.
All approaches have their strengths and weaknesses, however, each only allows to generate explanations supporting why an alternative was selected and not why an alternative was \emph{not} selected. We see this as an important drawback as procedural justice theory has shown that explanations of decisions are particularly important to maintain trust and legitimacy in the case of an adverse outcome~\cite{brockner1996integrative,colquitt2002explaining}.

Inspired by the literature on control and bribery~\cite{elkind2009swap,faliszewski2009llull}, which distinguishes between \emph{constructive bribery} where the goal is to change votes to make a losing candidate elected, and \emph{destructive bribery} where an elected candidate becomes unelected, we introduce the notion of \emph{destructive minimal supports} to answer the question ``Why was this candidate \emph{not} selected?'' with abductive reasoning.
To produce compact explanations for the election outcome, our approach uses tournaments which are a classical structure in social choice to compactly represent pairwise comparisons over a set of possible alternatives such as contending teams in sports, preferences expressed over candidates in an election or possible outcomes in a decision making problem (for an introduction see~\cite{tournamentHandbook2016,weight-tournamentHandbook2016}). 

Destructive minimal supports are partial sub-tournaments of the initial tournament in which the scrutinized losing candidate loses in all possible completions. That is, the candidate is a necessary loser of the sub-tournament, or, equivalently, not a possible winner. 
This problem can be seen as the dual of the Margin of Victory (MoV) problem~\cite{brill2020refining,doring2023margin} where one is looking for minimal changes in the tournament to change the outcome of the election. More generally, this work fits in the literature on the problem of finding necessary and possible winners in elections with partial information~\cite{aziz2015possible,konczak2005voting,lang2020collective,schwartz1966possible,xia2011determining}.

\noindent\textbf{Paper structure.} In Section~\ref{sec:prelim}, we recall classical notions on tournaments. 
In Section~\ref{sec:dms}, we define destructive minimal supports for the losing candidate $l$ for a tournament solution $S$ in a tournament $G$ as the set of inclusion minimal sub-tournaments of $G$ where $l$ is a necessary loser, i.e., loses in all completions of those sub-tournaments.
In Section~\ref{sec:pw}, we characterize partial tournaments where a specific candidate is a necessary loser, or, equivalently, not a possible winner for the following six tournament solutions : the maximin rule, the uncovered set and its weighted variant, the top-cycle, the Copeland rule, and the Borda rule.
Finally, in Section~\ref{sec:dmssize} we study the smallest destructive minimal supports, that is the destructive minimal supports comprised of the smallest amount of pairwise comparisons, we provide explicit formulas and bounds for their size, and we show how to obtain them efficiently.
See Table~\ref{tab:result} for a preview of our results.
\noindent All missing proofs are available in the extended version of this paper~\cite{contet2026characterizing}.

\begin{table}
\begin{minipage}{\textwidth}
\renewcommand{\thempfootnote}{\fnsymbol{mpfootnote}}
    \centering
    \begin{tabular}{ccc}
        \toprule
        \makecell[c]{Tournament\\ Solution} & \makecell[c]{Necessary Loser\\ Characterization} & \makecell[c]{Upper bound\\ on SdMS size} \\
        \midrule
        \makecell[c]{Top Cycle\\ $(\tc)$} & 
        {\makecell[c]{
        $\exists K\subseteq\candidates\setminus{\{l\}}$ s.t. $K\neq \varnothing$,\\$\forall c,c'\in K\times(\candidates\setminus{K})$, $\mu(c,c')=1$.\\
        (Theorem~\ref{th:tcchara})
        }} &
        \makecell[c]{$\left\lceil \frac{\m}{2} \right\rceil\left\lfloor \frac{\m}{2} \right\rfloor$\\ (Theorem~\ref{th:tcsize})}\\
        \midrule
        \makecell[c]{Borda\\ $(\borda)$} &
        {\makecell[c]{
        $\exists K\subseteq\candidates\setminus{\{l\}}$ s.t.\\
        $\displaystyle\min_{G'\in [G]} \mathop{{}\mathbb{E}}_{c\in K} \sigma_{\borda}(c,G') > \max_{G'\in [G]} \sigma_{\borda}(l,G')$.\\
        (Theorem~\ref{th:bochara}, Schwartz~\cite{schwartz1966possible})
        }} & 
        \makecell[c]{$\n(\m - 1)+1$\\ (Theorem~\ref{th:bosize})}\\
        \midrule
        \makecell[c]{Copeland\\ $(\cop)$} &
        {\makecell[c]{
        $\exists K\subseteq\candidates\setminus{\{l\}}$ s.t.\\
        $\displaystyle\min_{G'\in [G]} \mathop{{}\mathbb{E}}_{c\in K} \sigma_{\cop}(c,G') > \max_{G'\in [G]} \sigma_{\cop}(l,G')$.\\
        (Corollary of Theorem~\ref{th:bochara})
        }} & 
        \makecell[c]{$\m$\\ (Corollary~\ref{cor:cosize})}\\
        \midrule
        \makecell[c]{Maximin\\ $(\mm)$} & 
        {\makecell[c]{
        $\exists$ a tree $T=(K,E)$ s.t. $K\subseteq\candidates\setminus{\{l\}}$,\\$\forall c,c'\in K\times(\candidates\setminus{\{l,c\}})$, $\{c,c'\}\not\in E$\\$ \implies \mu(c,c')>\n - \max_{c''\in\candidates} \mu(c'',l)$.\\
        (Theorem~\ref{th:mmchara})
        }} &
        \makecell[c]{$\left\lceil \frac{\n+1}{2} \right\rceil(\m-2)+\n+1$ \\ (Theorem~\ref{th:mmsize})}\\
        \midrule
        \makecell[c]{Weighted\\ Uncovered Set\\ $(\wuc)$} & 
        {\makecell[c]{
        $\exists$ a tree $T=(K,E)$ s.t. $K\subseteq\candidates\setminus{\{l\}}$,\\$\forall c,c'\in K\times(\candidates\setminus{\{l,c\}})$, $\mu(c,l)\geq\frac{\n}{2}$,\\ $\{c,c'\}\not\in E \implies \mu(c,c') + \mu(c',l)\geq\n$.{\footnotesize *}\\
        (Theorem~\ref{th:wucchara})
        }} &
        \makecell[c]{$\n(\m-2)+\left\lceil \frac{\n+1}{2} \right\rceil$\\ (Theorem~\ref{th:wucsize})}\\
        \midrule
        \makecell[c]{Uncovered Set\\ $(\uc)$} & 
        {\makecell[c]{
        $\exists$ a tree $T=(K,E)$ s.t. $K\subseteq\candidates\setminus{\{l\}}$,\\$\forall c,c'\in K\times(\candidates\setminus{\{l,c\}})$, $\mu(c,l)=1$,\\ $\{c,c'\}\not\in E \implies \mu(c,c') + \mu(c',l)\geq1$.{\footnotesize *}\\
        (Corollary of Theorem~\ref{th:wucchara})
        }} &
        \makecell[c]{$\m-1$\\ (Corollary~\ref{cor:ucsize})}\\
        \bottomrule
        \multicolumn{3}{l}{\footnotesize * At least one of all the inequalities has to be strict.}\\
    \end{tabular}
    \caption{Overview of our results for a tournament $G=(\candidates,\mu)$ with $\n$ voters, $\m$ candidates in $\candidates$, and $l$ the losing candidate considered. 
    $\sigma_{\borda}(l,G)$ and $\sigma_{\cop}(l,G)$ are the Borda and Copeland scores of $l$ in $G$. $[G]$ is the set of all completions of $G$. All bounds are tight.}
    \label{tab:result}
\end{minipage}
\end{table}
\section{Preliminaries}\label{sec:prelim}

Let the \emph{number of voters} $n$ be a strictly positive integer. Aziz et al.~\cite{aziz2015possible} define a \emph{partial $n$-weighted tournament} as a pair $G = (\candidates, \mu)$ where $\candidates$ is a nonempty finite set of candidates and  $\mu : \candidates \times \candidates \to \{0, \dots, n\}$ a weight function such that for all distinct $x, y \in \candidates$, $\mu(x, y) + \mu(y, x) \leq n$ and $\mu(x,x)=0$.
A \emph{(complete) $n$-weighted tournament} satisfies for all distinct $x, y \in \candidates$, $\mu(x, y) + \mu(y, x) = n$.
Naturally, an (unweighted) tournament is a $1$-weighted tournament.

\begin{example}\label{ex:tournament}
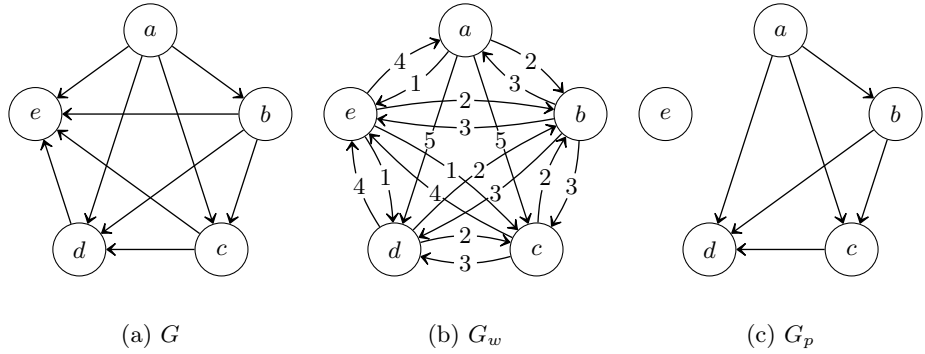
\begin{figure}
    \centering
    \begin{subfigure}{0.315\linewidth}
        \centering
        \begin{tikzpicture}[
align=center,
node distance=2cm and 2cm,
nodes={align=center},
inner frame sep = 0cm]
    \node [node] (a) at (90:\rp) {$a$};
    \node [node] (b) at (18:\rp) {$b$};
    \node [node] (c) at (-54:\rp) {$c$};
    \node [node] (d) at (-126:\rp) {$d$};
    \node [node] (e) at (-198:\rp) {$e$};

    \draw [arrow] (a) to (b) ;
    \draw [arrow] (a) to (c) ;
    \draw [arrow] (a) to (d) ;
    \draw [arrow] (a) to (e) ;

    \draw [arrow] (b) to (c) ;
    \draw [arrow] (b) to (d) ;
    \draw [arrow] (b) to (e) ;

    \draw [arrow] (c) to (e) ;
    \draw [arrow] (c) to (d) ;
    
    \draw [arrow] (d) to (e) ;
\end{tikzpicture}
        \caption{$G$}\label{fig:exunw}
    \end{subfigure}
    ~
    \begin{subfigure}{0.315\linewidth}
        \centering
        \begin{tikzpicture}[
    align=center,
    node distance=2cm and 2cm,
    nodes={align=center},
    inner frame sep = 0cm]
    \node [node] (a) at (90:\rp) {$a$};
    \node [node] (b) at (18:\rp) {$b$};
    \node [node] (c) at (-54:\rp) {$c$};
    \node [node] (d) at (-126:\rp) {$d$};
    \node [node] (e) at (-198:\rp) {$e$};

    \draw [arrow] (a) to [out=-21, in=129] node[quotes] {$2$} (b) ;
    \draw [arrow] (a) to node[quotes, inner sep=1pt] {$5$} (c) ;
    \draw [arrow] (a) to node[quotes, inner sep=1pt] {$5$} (d) ;
    \draw [arrow] (a) to [out=-129, in=21] node[quotes] {$1$} (e) ;
    
    \draw [arrow] (b) to [out=159, in=-51] node[quotes] {$3$} (a) ;
    \draw [arrow] (b) to [out=-93, in=57] node[quotes] {$3$} (c) ;
    \draw [arrow] (b) to [out=-134, in=26] node[quotes, inner sep=1pt] {$3$} (d) ;
    \draw [arrow] (b) to [out=190, in=-10] node[quotes, inner sep=1pt] {$3$} (e) ;
    
    \draw [arrow] (c) to [out=87, in=-123] node[quotes] {$2$} (b) ;
    \draw [arrow] (c) to [out=195, in=-15] node[quotes] {$3$} (d) ;
    \draw [arrow] (c) to [out=154, in=-46] node[quotes, inner sep=1pt] {$4$} (e) ;
    
    \draw [arrow] (d) to [out=46, in=-154] node[quotes, inner sep=1pt] {$2$} (b) ;
    \draw [arrow] (d) to [out=15, in=165] node[quotes] {$2$} (c) ;
    \draw [arrow] (d) to [out=123, in=-87] node[quotes] {$4$} (e) ;
    
    \draw [arrow] (e) to [out=51, in=-159] node[quotes] {$4$} (a) ;
    \draw [arrow] (e) to [out=10, in=170] node[quotes, inner sep=1pt] {$2$} (b) ;
    \draw [arrow] (e) to [out=-26, in=134] node[quotes, inner sep=1pt] {$1$} (c) ;
    \draw [arrow] (e) to [out=-57, in=93] node[quotes] {$1$} (d) ;
\end{tikzpicture}
        \caption{$G_w$}\label{fig:exw}
    \end{subfigure}
    ~
    \begin{subfigure}{0.315\linewidth}
    \centering
    \begin{tikzpicture}[
align=center,
node distance=2cm and 2cm,
nodes={align=center},
inner frame sep = 0cm]
    \node [node] (a) at (90:\rp) {$a$};
    \node [node] (b) at (18:\rp) {$b$};
    \node [node] (c) at (-54:\rp) {$c$};
    \node [node] (d) at (-126:\rp) {$d$};
    \node [node] (e) at (-198:\rp) {$e$};

    \draw [arrow] (a) to (b) ;
    \draw [arrow] (a) to (c) ;
    \draw [arrow] (a) to (d) ;

    \draw [arrow] (b) to (c) ;
    \draw [arrow] (b) to (d) ;

    \draw [arrow] (c) to (d) ;
\end{tikzpicture}
    \caption{$G_p$}
    \label{fig:nl}
    \end{subfigure}
    \caption{An (unweighted) tournament G (\subref{fig:exunw}), a 5-weighted tournament $G_w$(\subref{fig:exw}), and an (unweighted) partial tournament $G_p$ (\subref{fig:nl}). Throughout this paper, edges with a weight of $0$ are not shown and edge labels are omitted in the unweighted case.}
    \label{fig:tournamentsex}
\end{figure}
In Figure~\ref{fig:exunw}, the edge of $G$ between the pair of alternatives $a$ and $b$ shows that $a$ is preferred to $b$. In Figure~\ref{fig:exw}, the edge from $a$ to $b$ with a weight of $2$ in $G_w$ shows that $a$ is preferred to $b$ by $2$ voters. Note also that $G_p$ is a sub-tournament of $G$ as all edges in $G_p$ are in $G$.

\end{example}

Given a tournament $G$ it is then possible to select a set of desirable candidates using a \emph{tournament solution} $S$ (also called tournaments rule) which is a function taking a tournament as input and returning the winners set as a non-empty subset of candidates noted $S(G)$.
We focus on the following set of classical tournament solutions.

\begin{itemize}
    \item The \emph{top cycle} ($\tc$) is the unique minimal dominant nonempty subset of candidates of an unweighted tournament $G$, where a nonempty subset of candidates $A \subseteq \candidates$ is called \emph{dominant} in $G=(\candidates,\mu)$ if for each candidates $x\in A$ and $y\in  \candidates\setminus{A}$, $\mu(x,y)=1$.
    Alternatively, $\tc$ can also be defined as the set of candidates that can reach every other candidate via a directed path in $G$ or, in terms of graph theory, as the unique (since $G$ is complete) top strongly connected component of $G$. 
    \item The \emph{Borda score} of a candidate $c\in\candidates$ in a weighted tournament $G=(\candidates, \mu)$ is $\sigma_{\borda}(c,G) = \sum_{c'\in\candidates}\mu(c,c')$. The \emph{Borda rule} ($\borda$) selects candidates that have a maximal Borda score.
    \item The \emph{Copeland score} of a candidate $c\in\candidates$ in an unweighted tournament $G = (\candidates, \mu)$ is $\sigma_{\cop}(c,G) = |\{c':c'\in \candidates, \mu(c,c')=1\}|$. The \emph{Copeland rule} ($\cop$) selects candidates that have a maximal Copeland score. It can be seen as the restriction of $\borda$ to unweighted tournaments.
    \item The \emph{maximin score} of a candidate $c\in\candidates$ in a weighted tournament $G=(\candidates, \mu)$ is $\sigma_{\mm}(c,G) = \min_{c'\in \candidates\setminus{\{c\}}} \mu(c,c')$. The \emph{maximin rule} ($\mm$) selects candidates that have a maximal maximin score.
    \item The \emph{weighted uncovered set} ($\wuc$) is the nonempty subset of candidates that are not strictly weighted covered by any other candidate in a weighted tournament $G=(\candidates,\mu)$ where a candidate $x\in\candidates$ is said to strictly \emph{cover} another candidate $y\in\candidates$ if for all $z\in\candidates$, $\mu(x,z)\geq \mu(y,z)$ and at least one inequality is strict. For unweighted tournaments, we simply talk about the \emph{uncovered set} ($\uc$).
\end{itemize}

%

\section{(Smallest) Destructive Minimal Supports for Tournaments}\label{sec:dms}

To explain why a candidate $c$ loses a tournament $G$, abductive reasoning extracts a subset minimal set of features of the tournament $G$ ensuring that the scrutinized candidate loses. We call destructive minimal supports such minimal sub-tournaments of $G$ where $c$ loses independently of the rest of the tournament.
Given two partial tournaments $G=(\candidates,\mu)$ and $G'=(\candidates,\mu')$, we say that $G'$ is an \emph{extension} of $G$ denoted $G \subseteq G'$ if for all distinct $x, y \in \candidates$, $\mu(x,y) \leq \mu'(x,y)$ (for instance $G_p\subseteq G$ in Example~\ref{ex:tournament}). 
We refer to $[G]$ as the set of all complete extensions of $G$.


We introduce the notion of \emph{necessary losers} to formalize this idea of a candidate losing in all completions of a partial tournament. This definition can be viewed as a parallel to the notion of \emph{necessary winners}, i.e., candidates winning no matter what, and is equivalent to not being a \emph{possible winner} or a candidate that can win in at least one completion (see Konczak and Lang~\cite{konczak2005voting}).

\begin{definition}\label{def:nl}
    Given $\n$ voters, a partial $\n$-weighted tournament $G$, a tournament solution $S$  and a candidate $c \in \candidates$, $c$ is a \emph{necessary loser} for $G$ (w.r.t $S$) if for every completion $G'\in [G]$ we have that $c \not \in S(G')$. We write $c \in \NL{S}{G}$.
\end{definition}

\begin{example}\label{ex:partial}
Consider the partial tournament in Figure~\ref{fig:nl} and the Copeland rule. Since $a$ already secured $3$ wins, $c$ and $d$ cannot catch back as their Copeland score is respectively $1$ and $0$ and they can only improve it by $1$ by being preferred to $e$. On the contrary, $a$ can win, typically if $e$ is not preferred to him. $b$ can win if $e$ is preferred to $a$ and $b$ is preferred to $e$. $e$ can win as it could be a Condorcet winner for instance. Hence, $c$ and $d$ are the necessary Copeland-losers for $G$.
\end{example}

We now define destructive minimal supports for tournaments.

\begin{definition}\label{def:dms}
Given $\n$ voters, an $\n$-weighted tournament $G = (\candidates, \mu)$, a tournament solution $S$ and a losing candidate $l \in \candidates\setminus{S(G)}$, a \emph{destructive minimal support (dMS)} for $l \not\in S(G)$ is a partial tournament $G'\subseteq G$ such that: 
\begin{enumerate}
    \item[(a)]  $l \in \NL{S}{G'}$
    \item[(b)] $G'$ is $\subseteq$-minimal, i.e., all partial tournaments $G''\subsetneq G'$ are such that $l \not\in \NL{S}{G''}$.
\end{enumerate} 
\end{definition}

\begin{example}\label{ex:dms}
\begin{figure}
    \centering
    \begin{subfigure}{0.315\linewidth}
        \centering
        \begin{tikzpicture}[
align=center,
node distance=2cm and 2cm,
nodes={align=center},
inner frame sep = 0cm]
    \node [node] (a) at (90:\rp) {$a$};
    \node [node] (b) at (18:\rp) {$b$};
    \node [node] (c) at (-54:\rp) {$c$};
    \node [node] (d) at (-126:\rp) {$d$};
    \node [node] (e) at (-198:\rp) {$e$};

    \draw [arrow] (a) to (b) ;
    \draw [arrow] (a) to (c) ;
    \draw [arrow] (a) to (d) ;

    \draw [arrow] (b) to (c) ;
    \draw [arrow] (b) to (d) ;
    \draw [arrow] (b) to (e) ;

    \draw [arrow] (c) to (d) ;
    \draw [arrow] (c) to (e) ;

    \draw [arrow] (d) to (e) ;

    \draw [arrow] (e) to (a) ;
\end{tikzpicture}
        \caption{$G$}\label{fig:dms1}
    \end{subfigure}
    ~
    \begin{subfigure}{0.315\linewidth}
        \centering
        \begin{tikzpicture}[
align=center,
node distance=2cm and 2cm,
nodes={align=center},
inner frame sep = 0cm]
    \node [node] (a) at (90:\rp) {$a$};
    \node [node] (b) at (18:\rp) {$b$};
    \node [node] (c) at (-54:\rp) {$c$};
    \node [node] (d) at (-126:\rp) {$d$};
    \node [node] (e) at (-198:\rp) {$e$};

    \draw [arrow] (a) to (b) ;
    \draw [arrow] (a) to (c) ;
    \draw [arrow] (a) to (d) ;

    \draw [arrow] (c) to (d) ;
\end{tikzpicture}
        \caption{$\X$}\label{fig:dms2}
    \end{subfigure}
    ~
    \begin{subfigure}{0.315\linewidth}
        \centering
        \begin{tikzpicture}[
align=center,
node distance=2cm and 2cm,
nodes={align=center},
inner frame sep = 0cm]
    \node [node] (a) at (90:\rp) {$a$};
    \node [node] (b) at (18:\rp) {$b$};
    \node [node] (c) at (-54:\rp) {$c$};
    \node [node] (d) at (-126:\rp) {$d$};
    \node [node] (e) at (-198:\rp) {$e$};

    \draw [arrow] (a) to (d) ;
    \draw [arrow] (b) to (d) ;
    \draw [arrow] (c) to (d) ;
\end{tikzpicture}
        \caption{$\Y$}\label{fig:dms3}
    \end{subfigure}
    \caption{A tournament $G$ (\subref{fig:dms1}), and $\X$ and $\Y$, two dMSs for $d\not\in\cop(G)$ (\subref{fig:dms2}) and (\subref{fig:dms3}).}
    \label{fig:dms}
\end{figure}
Consider the Copeland rule and the tournament $G$ in Figure~\ref{fig:dms1}. Clearly, $c\not\in\cop(G)$. To produce a dMS for $d\not\in\cop(G)$ like $\X$ in Figure~\ref{fig:dms2}, one has to find a partial sub-tournament of $G$ such that there always exists a candidate with a better Copeland score than $d$ independently of the way $G$ is completed. In $\X$, this is achieved by using a mix of pairwise comparisons lost by $d$ and won by a stronger candidate $a$. Indeed, at best $d$ can achieve a score of $2$ while $a$ at least a score of $3$. This sub-tournament is subset minimal as removing a defeat from $d$ breaks the previous reasoning. The case of $\Y$ in Figure~\ref{fig:dms3} is similar. Here it is sufficient to attribute $3$ losses to $d$ as it is impossible to complete the rest of the tournament without giving at least a score of $2$ to another candidate. $d$ loses a strict majority of its comparisons and there always exists a candidate in the tournament that wins a majority of them by the pigeonhole principle.
\end{example}

As seen in the previous example, there can exist multiple dMSs for a given losing candidate. We focus on the simpler ones in line with Grice's manner criterion~\cite{grice1975logic} which suggests to pick the briefest explanation. To measure the dMSs size, we use total number of pairwise comparisons in the partial tournament in line with Contet et al.~\cite{contet2026explaining} and the microbribery setting of Faliszewski et al.~\cite{faliszewski2009llull}.

\begin{definition}\label{def:sdms}
    Given $n$ voters, an $n${-}weighted tournament $G{=}(\candidates,\mu)$, and a losing candidate $l \in \candidates\setminus{S(G)}$, we define the size of a dMS $\X{=}(\candidates,\mu_\X)$ for $l \not\in S(G)$ as $|\X| = \sum_{(c,c') \in \candidates^2} \mu_\X(c,c')$. $\X$ is a \emph{smallest destructive minimal support (SdMS)} for $l \not\in S(G)$ if and only if for all dMSs $\Y$ for $l \not\in S(G)$ we have $|\X| \leq |\Y|$.
\end{definition}

\section{Characterizing Necessary Losers}\label{sec:pw}

\subsection{Top Cycle}\label{sec:pwtc}

The characterization of top cycle is relatively straightforward. Essentially, a candidate $l$ can be a necessary loser if and only if there always exists a strongly connected component inaccessible from $l$ in each completion. For it to be the case, a ``one-way frontier'' has to split the initial partial tournament in two.

\begin{restatable}{theorem}{tcchara}\label{th:tcchara}
    Given a partial tournament $G =(\candidates,\mu)$, a candidate $l \in \candidates$ is a necessary $\tc$-loser for $G$ if and only if there exists $K\subseteq\candidates\setminus{\{l\}}$ with $K\neq \varnothing$ such that for each pair $(c,c')\in K\times(\candidates\setminus{K})$, $\mu(c,c')=1$.
\end{restatable}

\subsection{Borda}\label{sec:pwbo}
 
Schwartz proved the characterization for the Borda rule in the 1960s~\cite{schwartz1966possible}. It was originally expressed with flow networks so we translate it here in our framework.
 
 \begin{restatable}{theorem}{bochara}\label{th:bochara}
    \textbf{Schwartz~\cite{schwartz1966possible}.} Given $\n$ voters and a partial $\n$-weighted tournament $G =(\candidates,\mu)$ with $|\candidates| = \m$, a candidate $l \in \candidates$ is a necessary Borda loser for $G$ if and only if there exists $K \subseteq \candidates\setminus{\{l\}}$ such that
    $$\min_{G'\in [G]} \mathop{{}\mathbb{E}}_{c\in K} \sigma_{\borda}(c,G') > \max_{G'\in [G]} \sigma_{\borda}(l,G').$$
\end{restatable}

\subsection{Maximin and Weighted Uncovered Set}\label{sec:pwmmwuc}

In this section we provide characterizations of partial tournaments where a specific candidate $l$ is a necessary loser for the maximin rule and the weighted uncovered set rule. 
We proceed by reducing the problem of determining if $l$ is a necessary loser in a tournament $G$ to the existence of a perfect matching in a specific bipartite graph that we call the $l$-deficient incidence graph of $G$. This approach generalizes an idea used by Aziz et al. to prove that the possible winner problem for maximin is in P (Theorem 11 in~\cite{aziz2015possible}). We then use Hall's marriage theorem~\cite{hall1987representatives} which characterizes bipartite graphs admitting perfect matchings to characterize partial tournaments where $l$ is a necessary loser.

Let $B=(X,Y,Z)$ be a bipartite graph where $X$ and $Y$ are vertex sets and $Z$ the edge set.
For each subset $W\subseteq X$, the \emph{neighborhood} of $W$ in $B$, noted $N_{B}(W)$ is the subset of vertices in $Y$ adjacent in $B$ to vertices in $W$.
An \emph{\mbox{$X$-perfect} matching} (or simply a \emph{perfect matching}) is a subset of edges $M\subseteq Z$ such that each vertex of $X$ is adjacent to exactly one edge in $M$.

We distinguish the directed edge from $x$ to $y$, noted with parentheses $(x,y)$, from the (undirected) edge between $x$ and $y$, noted with braces $\{x,y\}$.

\begin{definition}
    Given a directed graph $G=(V,E)$ where $V$ is the set of vertices, $E$ the set of edges and a vertex $v\in V$, we call the \emph{$v$-deficient incidence graph} of $G$ the bipartite graph $I_G^v=(X,Y,Z)$ where $X=V\setminus{\{v\}}$, $Y = \{\{x,y\} : x,y\in V^2,\,(x,y)\in E\}$ and $Z=\{\{x,\{x,y\}\}: x,y\in V^2,\, (x,y)\in E\}$, i.e., a vertex $x\neq v$ is adjacent to an edge $e$ in $I_G^v$ if and only if $e$ is an edge leaving $x$ in $G$.
\end{definition}



\begin{example}\label{ex:incidence}
\begin{figure}
    \centering
    \begin{subfigure}{0.6\linewidth}
        \centering
        \begin{tikzpicture}[
align=center,
node distance=2cm and 2cm,
nodes={align=center},
inner frame sep = 0cm]
    \node [node] (a) at (90:\rg) {$a$};
    \node [node] (b) at ($(a) + (-60:\rg)$) {$b$};
    \node [node] (c) at ($(a) + (-120:\rg)$) {$c$};
    \draw [arrow] (a) to node {} (b) ;
    \draw [arrow] (b) to node {} (c) ;
    \draw [arrow] (c) to node {} (a) ;
    
    \node [node] (e) at (2.2,-1) {$e$};
    \node [node] (d) at ($(e) + (90:\rg)$) {$d$};
    \node [node] (f) at ($(e) + (-30:\rg)$) {$f$};
    \node [node] (g) at ($(e) + (210:\rg)$) {$g$};
    \draw [arrow] (e) to [out=80, in=-80] node {} (d) ;
    \draw [arrow] (d) to [out=-100, in=100] node {} (e) ;
    \draw [arrow] (e) to [out=-40, in=160] node {} (f) ;
    \draw [arrow] (f) to [out=140, in=-20] node {} (e) ;
    \draw [arrow] (e) to  [out=-160, in=40] node {} (g) ;
    \draw [arrow] (g) to  [out=20, in=-140] node {} (e) ;

    \node [node] (i) at (4.5,0.5) {$i$};
    \node [node] (h) at ($(i) + (140:\rg)$) {$h$};
    \node [node] (j) at ($(i) + (-40:\rg)$) {$j$};
    \draw [arrow] (i) to [out=130, in=-30] node {} (h) ;
    \draw [arrow] (h) to [out=-50, in=150] node {} (i) ;
    \draw [arrow] (i) to node {} (j) ;
\end{tikzpicture}
        \caption{$G$}\label{fig:ig1}
    \end{subfigure}
    ~
    \begin{subfigure}{0.3\linewidth}
        \centering
        \begin{tikzpicture}[
    align=center,
    node distance=0.cm and 1.0cm,
    inner frame sep = 0cm]
    \node (a) {$a$};
    \node [below= of a] (b) {$b$};
    \node [below= of b] (c) {$c$};
    \node [below= of c] (d) {$d$};
    \node [below= of d] (e) {$e$};
    \node [below= of e] (f) {$f$};
    \node [below= of f] (g) {$g$};
    \node [below= of g] (h) {$h$};
    \node [below= of h] (i) {$i$};

    \node [right= of a] (ab) {$\{a,b\}$};
    \node at ($(b-|ab)$) (ac) {$\{a,c\}$};
    \node at ($(c-|ab)$) (bc) {$\{b,c\}$};
    
    \coordinate (tempde) at ($(d)!.5!(e)$);
    \node at ($(tempde-|ab)$) (de) {$\{d,e\}$};
    \coordinate (tempef) at ($(e)!.5!(f)$);
    \node at ($(tempef-|ab)$) (ef) {$\{e,f\}$};
    \coordinate (tempeg) at ($(f)!.5!(g)$);
    \node at ($(tempeg-|ab)$) (eg) {$\{e,g\}$};

    \node at ($(h-|ab)$) (hi) {$\{h,i\}$};
    \node at ($(i-|ab)$) (ij) {$\{i,j\}$};

    \draw (a) to (ab);
    \draw (b) to (bc);
    \draw (c) to (ac);
    \draw (d) to (de);
    \draw (e) to (de);
    \draw (e) to (ef);
    \draw (e) to (eg);
    \draw (f) to (ef);
    \draw (g) to (eg);
    \draw (h) to (hi);
    \draw (i) to (hi);
    \draw (i) to (ij);
\end{tikzpicture}
        \caption{$I_G^j$}\label{fig:ig2}
    \end{subfigure}
    \caption{A directed graph $G$ (\subref{fig:ig1}) and its $j$-deficient incidence graph $I_G^j$ (\subref{fig:ig2}).}
    \label{fig:ig}
\end{figure}
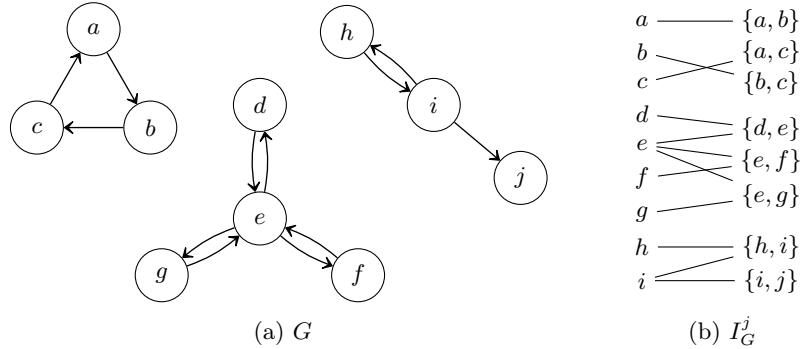
    
Consider the graph $G=(V,E)$ in Figure~\ref{fig:ig1}. To build its $j$-deficient incidence graph $I_G^j=(X,Y,Z)$ in Figure~\ref{fig:ig2} simply take the original set of vertices without $j$, $X=V\setminus{\{j\}}=\{a,b,c,d,e,f,g,h,i\}$, add each pair of vertices containing an edge in the original graph $Y=\{\{x,y\} :x,y\in V^2,\,(x,y)\in E\}$ and then link vertices with their out-going edges in $G$.
\end{example}

We identify graphs whose deficient incidence graph has no perfect matching.

\begin{restatable}{lemma}{perfectmatchingcharadef}\label{lem:perfectmatchingcharadef}
    Given a directed graph $G=(V,E)$ where $V$ is the set of vertices and $E$ the set of edges, and a vertex $v\in V$, the $v$-deficient incidence graph of $G$, $I_G^v=(X,Y,Z)$, 
    does not have an $X$-perfect matching if and only if there exists a tree $T=(K,E_T)$ with $K\subseteq V\setminus{\{v\}}$ such that for each pair $(c,c')\in K\times V$ with $c\neq c'$, if $\{c,c'\}\not\in E_T$ then $(c,c')\not\in E$.
\end{restatable}



\begin{proof}
Given a directed graph $G=(V,E)$ where $V$ is the set of vertices and $E$ the set of edges, a vertex $v\in V$ and the $v$-deficient incidence graph of $G$ $I_G^v=(X,Y,Z)$.
Given a subgraph $S\subseteq G$, let $V_S$ be the set of vertices contained in $S$ and $E_S$ the set of edges. We define the undirected version of $S$, $\widetilde{S}=(V_S, \widetilde{E_S})$ where $\widetilde{E_S}=\{\{c,c'\}:c,c'\in V_S^2,\, (c,c')\in E_S\}$.
To prove Lemma~\ref{lem:perfectmatchingcharadef}, we use a classical result from matching theory.
\smallskip

\noindent\textbf{Hall's marriage theorem~\cite{hall1987representatives}.}
Let $G=(X,Y,Z)$ be a finite bipartite graph with bipartite sets $X$ and $Y$ and edge set $Z$. 
An $X$-perfect matching exists if and only if for each subset $W\subseteq X$, $|W|\leq |N_G(W)|$.
\smallskip

($\implies$) Suppose $I_G^v$ has no $X$-perfect matching.
Then, according to Hall's marriage theorem, there exists a subset $W\subseteq X$, $|W|> |N_{I_G^v}(W)|$.
Given $T=(W,N_{I_G^v}(W))$,
if no connected component of $T$ is a tree then each connected component has more edges than vertices and $|W| \leq |N_{I_G^v}(W)|$.
This is not the case so exists $W'\subseteq W$ such that $T'=(W',N_{I_G^v}(W'))$ is a tree.
Since $W'\subseteq W$, and $v\not \in W$, $v\not \in W'$.
By definition of $I_G^v$, for each pair $(c,c')\in W'\times V$ with $c\neq c'$, if $\{c,c'\}\not\in N_{I_G^v}(W')$ then, since $c\in W'$, $(c,c')\not\in E$.

($\impliedby$) Suppose there exists a tree $T=(K,E_K)$ with $K\subseteq V\setminus{\{v\}}$ such that for each pair $(c,c')\in K\times V$ with $c\neq c'$, if $\{c,c'\}\not\in E_T$ then $(c,c')\not\in E$.
As $T$ is connected and acyclic $|K| = |\widetilde{E_K}| + 1$ and 
since $v\not\in K$ and $T$ is connected, $\widetilde{E_K}=N_{I_G^v}(K)$. 
Thus, $|K| = |N_{I_G^v}(K)| + 1$.
Additionally, since $v\not\in S$, $K \subseteq X$.
Hence according to Hall's marriage theorem, there is no $X$-perfect matching in $I_G^v$. \qed
\end{proof}

To illustrate this result we revisit the previous example to see why it does not admit an $X$-perfect matching.

\begin{example}\label{ex:perfectmatching}
Looking back at the previous graph $G=(V,E)$ in Figure~\ref{fig:ig1} and its $j$-deficient incidence graph $I_G^j=(X,Y,Z)$ in Figure~\ref{fig:ig2}, by considering each connected component of $G$, it becomes clear that no $X$-perfect matching exists because of the strongly connected component $\{d,e,f,g\}$ which has strictly less edges than vertices since it is a tree. Note that even if the connected component $\{h,i,j\}$ is also a tree, $j$ brings an additional possible edge to match without adding a vertex in the bipartite graph.
\end{example}

Using the previous result we can show that a candidate is a necessary maximin loser if and only if there exists a specific tree in the partial tournament.

\begin{restatable}{theorem}{mmchara}\label{th:mmchara}
    Given $\n$ voters and a partial $\n$-weighted tournament $G =(\candidates,\mu)$, a candidate $l \in \candidates$ is a necessary $\mm$-loser for $G$ if and only if there exists a tree $T=(K,E)$ with $K\subseteq\candidates\setminus{\{l\}}$ such that for each pair $(c,c')\in K\times\candidates$ with $c\neq c'$, if $\{c,c'\}\not\in E$ then $\mu(c,c')>\n - \max_{c''\in\candidates} \mu(c'',l)$.
\end{restatable}


\begin{proof}
    Given $\n$ voters, a partial $\n$-weighted tournament $G =(\candidates,\mu)$ with $|\candidates| = \m$, and a candidate $l \in \candidates$.
    If $l$ wins all its unspecified comparisons in $G$, $l$ achieves a maximin score of $t=\n-\max_{c\in\candidates}\mu(c,l)$.

    If $t \geq \left\lceil\frac{\n}{2}\right\rceil$, by winning all its unspecified comparisons in $G$, $l$ can become a Condorcet winner and hence a maximin winner. Thus, $l$ is not a necessary loser.
    And since for each $c\in \candidates$, $\mu(c,l)\leq\left\lceil\frac{\n}{2}\right\rceil$ while $t \geq \left\lceil\frac{\n}{2}\right\rceil$, hence $\mu(c,l)\leq\n - \max_{c''\in\candidates}\mu(c'',l)$. Hence $K\subseteq\candidates\setminus{\{l\}}$ and there exists no candidate satisfying the condition in the theorem and thus no tree.

    In the rest of this proof we assume $t < \left\lceil\frac{\n}{2}\right\rceil$.
    Let the graph $H=(V_H,E_H)$ be such that $V_H=\candidates$ and $E_H=\{(c,c'):c,c'\in \candidates^2,\,c\neq l,\,c\neq c',\, \mu(c,c') \leq t\}$, and $I_H^l=(X,Y,Z)$ be the $l$-deficient incidence graph of $H$.

    If $c$ is matched with $\{c,c'\}$ in $I_H^l$, then $G$ can be completed in $G'$ where $\mu'(c,c') \leq t$ which ensures that the maximin score of $c$ is less than $t$. 
    Aziz et al.~\cite{aziz2015possible} essentially showed that $l$ is not a necessary loser if and only if there exists an $X$-perfect matching in $I_H^l$. We prove it here for completeness.

    $(\implies)$ Suppose that $l\not\in\NL{\mm}{G}$, then there exists a completion $G'=(\candidates,\mu')$ of $G$ where $l$ is a $\mm$-winner, i.e., where $l$ has the maximal $\mm$ score. 
    Thus, in $G'$, for all candidate $c$ distinct from $l$, there exists $c'\neq c$ such that $\mu'(c,c')\leq \n-\max_{c''\in\candidates}\mu'(c'',l) \leq \n-\max_{c''\in\candidates}\mu(c'',l) = t$.
    Additionally, since $t < \left\lceil\frac{\n}{2}\right\rceil$, $2t<\n$ and $\{c,c'\}\in Y$ cannot force both the maximin score of $c$ and $c'$ to be below $t$. Hence, in $I_H^l$, each vertex in $Y$ can only be matched to at most one vertex of $X$. This results in an $X$-perfect matching.
    
    $(\impliedby)$ Suppose there exists an $X$-perfect matching in $I_H^l$, then for each candidate $c\in X$, exists $c'\in\candidates$ such that $c$ is matched with $\{c,c'\}$. Now, let $G'=(\candidates,\mu')$ be a completion of $G$ where $l$ wins all its remaining comparisons, i.e., such that for all $c\in\candidates$, $\mu'(l,c)=\n -\mu(c,l)$, where $\mu'(c,c')=\n-\max_{c''\in\candidates}\mu(c'',l)$ if $\{c,\{c,c'\}\}$ is part of the matching and where the rest of $G$ is completed arbitrarily. 
    Each candidate distinct from $l$ has a maximin score of at most $\n-\max_{c''\in\candidates}\mu(c'',l)$ and $l$ has a maximin score of $\min_{c''\in\candidates}\mu'(l,c'')=\min_{c''\in\candidates}(\n -\mu(c'',l)) = \n-\max_{c''\in\candidates}\mu(c'',l)$. Hence, $l\in\mm(G')$ and $l\not\in\NL{\mm}{G}$.

    Applying Lemma~\ref{lem:perfectmatchingcharadef}, we have that $l$ is not a necessary $\mm$-loser of $G$ if and only if there exists a tree $T=(K,E_T)$ with $K\subseteq V\setminus{\{v\}}$ such that for each pair $(c,c')\in K\times V$ with $c\neq c'$, if $\{c,c'\}\not\in E_T$ then $(c,c')\not\in E_H$.
    Finally, by construction, such tree exists if and only if there exists a tree $T=(K,E)$ in $G$ with $K\subseteq\candidates\setminus{\{l\}}$ such that for each pair $(c,c')\in K\times\candidates$ with $c\neq c'$, if $\{c,c'\}\not\in E$ then $\mu(c,c')>\n - \max_{c''\in\candidates}(c'',l)$. \qed
\end{proof}


\begin{example}\label{ex:mm}
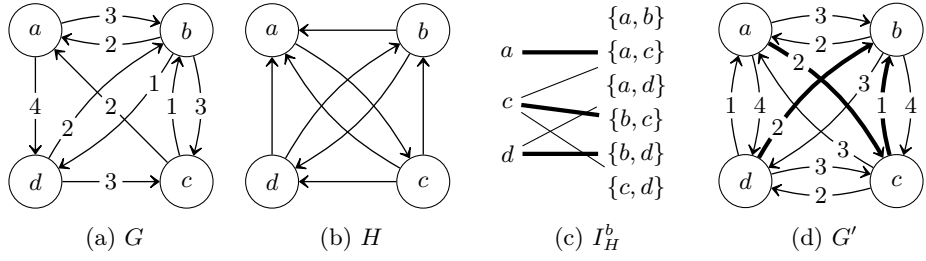
\begin{figure}
    \centering
    \begin{subfigure}{0.23\linewidth}
        \centering
        \raisebox{-0.5\height}{\begin{tikzpicture}[align=center]
    \node [node] (a) at (0,0) {$a$};
    \node [node] (b) at (2,0) {$b$};
    \node [node] (c)  at (2,-2) {$c$};
    \node [node] (d) at (0,-2) {$d$};

    \draw [arrow] (a) to [out=15, in=165] node[quotes]  {$3$} (b) ;
    \draw [arrow] (b) to [out=-165, in=-15] node[quotes]  {$2$} (a) ;

    \draw [arrow] (c) to node[quotes,pos=0.5] {$2$} (a) ; 
    
    \draw [arrow] (a) to node[quotes] {$4$} (d) ;
    
    \draw [arrow] (b) [out=-75, in=75] to node[quotes] {$3$} (c) ;
    \draw [arrow] (c) [out=105, in=-105] to node[quotes] {$1$} (b) ;

    \draw [arrow] (b) [out=-120, in=30] to node[quotes,pos=0.2] {$1$} (d) ;
    \draw [arrow] (d) [out=60, in=-150] to node[quotes,pos=0.2] {$2$} (b) ;
    
    \draw [arrow] (d) to  node[quotes]  {$3$} (c) ; 
\end{tikzpicture}}
        \caption{$G$}\label{fig:mm1}
    \end{subfigure}
    ~
    \begin{subfigure}{0.23\linewidth}
        \centering
        \raisebox{-0.5\height}{\begin{tikzpicture}[align=center]
    \node [node] (a) at (0,0) {$a$};
    \node [node] (b) at (2,0) {$b$};
    \node [node] (c)  at (2,-2) {$c$};
    \node [node] (d) at (0,-2) {$d$};

    \draw [arrow] (b) to  node[]  {} (a) ; 

    \draw [arrow] (a) [out=-30, in=120] to node[pos=0.2] {} (c) ;
    \draw [arrow] (c) [out=150, in=-60] to node[pos=0.2] {} (a) ;
    
    \draw [arrow] (d) to node[] {} (a) ;
    
    \draw [arrow] (c) to node[] {} (b) ; 

    \draw [arrow] (b) [out=-120, in=30] to node[pos=0.2] {} (d) ;
    \draw [arrow] (d) [out=60, in=-150] to node[pos=0.2] {} (b) ;
    
    \draw [arrow] (c) to  node[]  {} (d) ;
\end{tikzpicture}}
        \caption{$H$}\label{fig:mm2}
    \end{subfigure}
    ~
    \begin{subfigure}{0.23\linewidth}
        \centering
        \raisebox{-0.5\height}{\begin{tikzpicture}[
    align=center,
    node distance=-0.1cm and 1.0cm,
    inner frame sep = 0cm]

    \node (ab) {$\{a,b\}$};
    \node [below= of ab] (ac) {$\{a,c\}$};
    \node [below= of ac] (ad) {$\{a,d\}$};
    \node [below= of ad] (bc) {$\{b,c\}$};
    \node [below= of bc] (bd) {$\{b,d\}$};
    \node [below= of bd] (cd) {$\{c,d\}$};

    \node [left= of ac] (a) {$a$};
    \node [left= of bd] (d) {$d$};
    \node at ($(a)!.5!(d)$) (c) {$c$};

    \draw [ultra thick] (a) to (ac);
    \draw (c) to (ac);
    \draw [ultra thick] (c) to (bc);
    \draw (c) to (cd);
    \draw (d) to (ad);
    \draw [ultra thick] (d) to (bd);
\end{tikzpicture}}
        \caption{$I_H^b$}\label{fig:mm3}
    \end{subfigure}
    ~
    \begin{subfigure}{0.23\linewidth}
        \centering
        \raisebox{-0.5\height}{\begin{tikzpicture}[align=center]
    \node [node] (a) at (0,0) {$a$};
    \node [node] (b) at (2,0) {$b$};
    \node [node] (c)  at (2,-2) {$c$};
    \node [node] (d) at (0,-2) {$d$};

    \draw [arrow] (a) to [out=15, in=165] node[quotes]  {$3$} (b) ;
    \draw [arrow] (b) to [out=-165, in=-15] node[quotes]  {$2$} (a) ;

    \draw [arrow, ultra thick] (a) [out=-30, in=120] to node[quotes,pos=0.2] {$2$} (c) ;
    \draw [arrow] (c) [out=150, in=-60] to node[quotes,pos=0.2] {$3$} (a) ;
    
    \draw [arrow] (a) [out=-75, in=75] to node[quotes] {$4$} (d) ;
    \draw [arrow] (d) [out=105, in=-105] to node[quotes] {$1$} (a) ;
    
    \draw [arrow] (b) [out=-75, in=75] to node[quotes] {$4$} (c) ;
    \draw [arrow, ultra thick] (c) [out=105, in=-105] to node[quotes] {$1$} (b) ;

    \draw [arrow] (b) [out=-120, in=30] to node[quotes,pos=0.2] {$3$} (d) ;
    \draw [arrow, ultra thick] (d) [out=60, in=-150] to node[quotes,pos=0.2] {$2$} (b) ;
    
    \draw [arrow] (d) to [out=15, in=165] node[quotes]  {$3$} (c) ;
    \draw [arrow] (c) to [out=-165, in=-15] node[quotes]  {$2$} (d) ;
\end{tikzpicture}}
        \caption{$G'$}\label{fig:mm4}
    \end{subfigure}
    \caption{A partial $5$-weighted tournament $G$ (\subref{fig:mm1}), its resulting $H$ graph (see proof of Theorem~\ref{th:mmchara} and Example~\ref{ex:mm}) (\subref{fig:mm2}), the $b$-deficient incidence graph of $H$ $I_H^b$ (\subref{fig:mm3}), and $G'$, a completion of $G$ where $b$ is a $\mm$-winner (\subref{fig:mm4}).}
    \label{fig:mm}
\end{figure}
    
Consider, in Figure~\ref{fig:mm}, the partial $5$-weighted tournament $G$. $H$ is obtained from $G$ by only keeping the edges with a weight smaller than $t=\n-\max_{x\in\candidates}\mu(x,b) = 5-3=2$. Edges in $H$ are edges which can potentially be used to limit the maximin score of the competitors of $b$. We then try to build a perfect matching in the $b$-deficient incidence graph of $H$. We show one with the thick edges. Finally, from this matching we produce $G'$, a completion of $G$ where $b$ is a $\mm$-winner, hence not a necessary loser in $G$. 
\end{example}

We proceed with the same idea for the weighted uncovered set.


\begin{restatable}{theorem}{wucchara}\label{th:wucchara}
    Given $\n$ voters and a partial $\n$-weighted tournament $G =(\candidates,\mu)$, a candidate $l \in \candidates$ is a necessary $\wuc$-loser for $G$ if and only if there exists a tree $T=(K,E)$ with $K\subseteq\candidates\setminus{\{l\}}$ such that for each $c\in K$, $\mu(c,l)\geq\frac{\n}{2}$ and for each $c'\in \candidates\setminus{\{l,c\}}$, if $\{c,c'\}\not\in E$ then $\mu(c,c') + \mu(c',l)\geq\n$.
    Additionally, at least one of all the previous inequalities has to be strict.
\end{restatable}


Besides the strict inequality constraint, the proof follows the same approach as the proof of Theorem~\ref{th:mmchara} with $H=(V_H,E_H)$ such that $V_H=\candidates$ and $E_H=\{(c,c'):c,c'\in\candidates^2,\,c\neq l,\,c\neq c',\,\mu(c,c') < \n - \mu(c',l)\}\cup\{(c,l):c\in\candidates,\,c\neq l,\, \mu(c,l) < \left\lceil\frac{\n}{2}\right\rceil\}$.

\section{Smallest Destructive Minimal Supports}\label{sec:dmssize}

In this section we continue our work with the analysis of the smallest destructive minimal supports. 
Proofs in this section all follow the same structure. First, we show how to build small destructive minimal supports in each sub-case and then we prove they match the lower bounds, hence that they are the smallest dMSs.
\medskip

\noindent\textbf{Top Cycle.}
To guarantee a necessary loser, we separate the top cycle and the loser. 
Such separation can only occurs between strongly connected components.
Since the order between them is transitive and complete, it is linear. 
Additionally, to split $\m$ candidates in two groups of sizes $k$ and $\m-k$, $\m(\m-k)$ comparisons are required. By concavity, the smallest frontier is either right after the top cycle (which is the top component) or right before the loser's component.

 \begin{restatable}{theorem}{tcsize}\label{th:tcsize}
    Given a complete tournament $G =(\candidates,\mu)$ with $|\candidates| = \m$, and a losing candidate $l \in \candidates\setminus{\tc(G)}$, for each SdMS $\X$ for $l \not \in \tc(G)$, we have $|\X|{=} \displaystyle \min_{t\in \{\alpha, \beta\}} t(\m-t)$ where $\alpha = |\tc(G)|$ and $\beta = |\{c : c\in\candidates, l \text{ can reach } c \text{ in } G\}|$.
    An SdMS $\X$ can be computed in polynomial time, and $|\X|\leq\left\lceil\frac{\m}{2}\right\rceil\left\lfloor\frac{\m}{2}\right\rfloor$.
\end{restatable}


\noindent\textbf{Borda.}
Because there always exists a candidate at the average or above, when possible, it is enough to show that the loser performs strictly below average (case~\ref{it:boi}). Else we have to present a coalition with a stronger average Borda score (case~\ref{it:boii}). The coalition of size $1$, which always exists, gives the upper bound. We conjecture that the problem of finding an SdMS for Borda is in P.

 \begin{restatable}{theorem}{bosize}\label{th:bosize}
   Given $\n$ voters, a complete $\n$-weighted tournament $G =(\candidates,\mu)$ with $|\candidates| = \m$, a losing candidate $l \in \candidates\setminus{\borda(G)}$ with Borda score $\sigma_{\borda}(l)$, for each SdMS $\X$ for $l \not \in \borda(G)$: 
   \begin{enumerate}[label=\textit{\roman*.}]
       \item if $\sigma_{\borda}(l) \leq \left\lfloor\frac{\n(\m-1)-1}{2}\right\rfloor$ then $|\X|{=}\left\lceil\frac{\n(\m-1)+1}{2}\right\rceil$\label{it:boi}
       \item else $\left\lceil\frac{\n(\m-1)+1}{2}\right\rceil<|\X|\leq \n(\m-1)+1 - \max_{c\in A} \mu(c,l)$ where $A=\{c:c\in\candidates, \sigma_{\borda}(c) > \sigma_{\borda}(l)\}$.\label{it:boii}
   \end{enumerate}
   For each SdMS $\X$, $|\X|\leq\n(\m-1)+1$.
\end{restatable}

\noindent\textbf{Copeland.}
The following result builds on the Theorem~\ref{th:bosize} as Copeland coincides with Borda when $\n=1$. This more constrained case allows for stronger results.

\begin{restatable}{corollary}{cosize}
    \label{cor:cosize}
    Given a complete tournament $G =(\candidates,\mu)$ with $|\candidates| = \m$, and a losing candidate $l \in \candidates\setminus{\cop(G)}$ with Copeland score $\sigma_{\cop}(l)$, for each SdMS $\X$ for $l \not \in \cop(G)$:
   \begin{enumerate}[label=\textit{\roman*.}]
       \item if $\sigma_{\cop}(l) < \left\lfloor\frac{\m-1}{2}\right\rfloor$ then $|\X|=\left\lceil\frac{\m}{2}\right\rceil$\label{it:coi}
       \item else if there exists $c\in \candidates$ such that $\mu(c,l)=1$ and $\sigma_{\cop}(c) > \sigma_{\cop}(l)$ then $|\X|=\m-1$
       \item else $\m-1\leq|\X|\leq\m$.
   \end{enumerate}
   For each SdMS $\X$, $|\X|\leq\m$.
\end{restatable}

\noindent\textbf{Maximin.} Here, the idea is to bound the maximin score of the losing candidate by keeping comparisons where he loses against one candidate while ensuring there exists a candidate with a higher maximin score by keeping enough comparisons where he wins against each other candidate.
Case \ref{it:mmii} of Theorem~\ref{th:mmsize} is when the two sets of comparisons overlap and case \ref{it:mmiii} is when they do not.
An SdMS fo maximin relies on a tree (in the sense of Theorem~\ref{th:mmchara}) reduced to one candidate. 

\begin{restatable}{theorem}{mmsize}\label{th:mmsize}
    Given $\n$ voters, a complete $\n$-weighted tournament $G =(\candidates,\mu)$ with $|\candidates| = \m$, and a losing candidate $l \in \candidates\setminus{\mm(G)}$ with maximin score $\sigma_{\mm}(l)$, for each SdMS $\X$ for $l \not \in \mm(G)$:
    \begin{enumerate}[label=\textit{\roman*.}]
        \item if $\m = 2$ then $|\X|=\left\lceil\frac{\n+1}{2}\right\rceil$
        \item else if there exists $c\in \candidates$ such that $\sigma_{\mm}(c) > \sigma_{\mm}(l)$ and $\mu(c,l)\geq \n - \sigma_{\mm}(l)$, then $|\X|{=} (\sigma_{\mm}(l)+1)(\m-3) + \n + 1$\label{it:mmii}
        \item else $|\X|{=} (\sigma_{\mm}(l)+1)(\m-2) + \n + 1$\label{it:mmiii}
    \end{enumerate}
    An SdMS $\X$ can be computed in polynomial time, and $|\X|\leq \left\lceil\frac{\n+1}{2}\right\rceil(\m-2) + \n + 1$.
\end{restatable}

\noindent\textbf{Weighted Uncovered Set.}
With the weighted uncovered set, a candidate has to cover the necessary loser in each possible completion.
One specific candidate can cover the losing candidate in all completions. Concretely, with $l$ the losing candidate and $c_0$ the ``necessary'' covering candidate, for each candidate $c$, $\mu(c_0,c)\geq \n - \mu(c,l)$ by fixing high enough $\mu(c_0,c)$ and $\mu(c,l)$ (case~\ref{it:wucii} in Theorem~\ref{th:wucsize}, dMS $\X$ in Example~\ref{ex:wuc}). The tree of Theorem~\ref{th:wucchara} then contains one candidate.
However, in specific instances, a tree with two candidates can yield smaller dMSs (case~\ref{it:wuci} in Theorem~\ref{th:wucsize}, dMS $\Y$ in Example~\ref{ex:wuc}).  
\begin{figure}[b]
\centering
\begin{subfigure}{0.315\linewidth}
    \centering
    \begin{tikzpicture}[
align=center,
node distance=2cm and 2cm,
nodes={align=center},
inner frame sep = 0cm]
    \node [node] (a) at (90:\r) {$a$};
    \node [node] (b) at (-30:\r) {$b$};
    \node [node] (c) at (210:\r) {$c$};

    \draw [arrow] (a) to [out=-40, in=100] node[quotes] {$3$} (b) ;
    \draw [arrow] (b) to [out=140, in=-80] node[quotes] {$2$} (a) ;

    \draw [arrow] (a) to [out=-100, in=40] node[quotes] {$3$} (c) ;
    \draw [arrow] (c) to [out=80, in=-140] node[quotes] {$2$} (a) ;

    \draw [arrow] (b) to [out=200, in=-20] node[quotes] {$3$} (c) ;
    \draw [arrow] (c) to [out=20, in=160] node[quotes] {$2$} (b) ;
\end{tikzpicture}
    \caption{$G$}\label{fig:wuc1}
\end{subfigure}
~
\begin{subfigure}{0.315\linewidth}
    \centering
    \begin{tikzpicture}[
align=center,
node distance=2cm and 2cm,
nodes={align=center},
inner frame sep = 0cm]
    \node [node] (a) at (90:\r) {$a$};
    \node [node] (b) at (-30:\r) {$b$};
    \node [node] (c) at (210:\r) {$c$};

    \draw [arrow] (a) to node[quotes] {$3$} (b) ;

    \draw [arrow] (a) to node[quotes] {$3$} (c) ;

    \draw [arrow] (b) to node[quotes] {$2$} (c) ;
\end{tikzpicture}
    \caption{$\X$}\label{fig:wuc2}
\end{subfigure}
~
\begin{subfigure}{0.315\linewidth}
    \centering
    \begin{tikzpicture}[
align=center,
node distance=2cm and 2cm,
nodes={align=center},
inner frame sep = 0cm]
    \node [node] (a) at (90:\r) {$a$};
    \node [node] (b) at (-30:\r) {$b$};
    \node [node] (c) at (210:\r) {$c$};

    \draw [arrow] (a) to node[quotes] {$3$} (c) ;

    \draw [arrow] (b) to node[quotes] {$3$} (c) ;
\end{tikzpicture}
    \caption{$\Y$}\label{fig:wuc3}
\end{subfigure}
\caption{A $5$-weighted tournament G (\subref{fig:wuc1}), two dMSs for $c{\not\in}\wuc(G)$ (\subref{fig:wuc2}) and (\subref{fig:wuc3}).}
\label{fig:wuc}
\end{figure}
\begin{restatable}{theorem}{wucsize}\label{th:wucsize}
    Given $\n$ voters, a complete $\n$-weighted tournament $G =(\candidates,\mu)$ with $|\candidates| = \m$, a losing candidate $l \in \candidates\setminus{\wuc(G)}$, for each SdMS $\X$ for $l \not \in \wuc(G)$:
    \begin{enumerate}[label=\textit{\roman*.}]
        
        \item if $p=\max_{c_0,c_1} \sum_{c\in\candidates\setminus{\{l,c_0,c_1\}}}\mu(c,l)-\n(\m-3)+\left\lfloor\frac{\n}{2}\right\rfloor>0$
        with $c_0,c_1{\in}(\candidates\setminus{\{l\}})^2$ such that $c_0\neq c_1$, for $i\in\{0,1\}$, $\mu(c_i,l)\geq \frac{\n}{2}$, for all $c\in\candidates\setminus{\{l,c_0,c_1\}}$, $\mu(c_i,c)+\mu(c,l)\geq\n$, and at least one of the inequalities is strict,
        
        then $|\X|{=} \n(\m-2) +\left\lceil\frac{\n+1}{2}\right\rceil -p$\label{it:wuci}
        
        
        \item else $|\X|{=} \n(\m-2) + \left\lceil\frac{\n+1}{2}\right\rceil$\label{it:wucii}
    \end{enumerate}
    An SdMS $\X$ can be computed in polynomial time, and $|\X|\leq \n(\m-2) + \left\lceil\frac{\n+1}{2}\right\rceil$.
\end{restatable}

We illustrate the non-trivial case \ref{it:wuci} with the following example.

\begin{example}\label{ex:wuc}
Consider the $5$-weighted tournament $G$ in Figure~\ref{fig:wuc1}. In the case of $\X$ in Figure~\ref{fig:wuc2}, $a$ weighted covers $c$ across all completions. However, we can allow the covering candidate to vary dependently of the completion as in the dMS $\Y$ in Figure~\ref{fig:wuc3}. If a majority of voters prefers $a$ to $b$ then $a$ covers $c$ else $b$ covers $c$. 
Here, $p=\sum_{x\in\candidates\setminus{\{a,b,c\}}}\mu(x,c)-\n(\m-3)+\left\lfloor\frac{\n}{2}\right\rfloor=0+0+2$. Hence $\Y$ contains $2$ less comparisons than $\X$ ($|\X|=8$, $|\Y|=6$).
\end{example}


\noindent\textbf{Uncovered Set.}
The analogous result for the uncovered set is simpler than Theorem~\ref{th:wucsize} since case~\ref{it:wuci} cannot occur as $\sum_{c\in\candidates\setminus{\{l,c_0,c_1\}}}\mu(c,l) \leq \m-3$ and $p\leq 0$.

 \begin{restatable}{corollary}{ucsize}\label{cor:ucsize}
    Given a complete tournament $G =(\candidates,\mu)$ with $|\candidates| = \m$, and a losing candidate $l \in \candidates\setminus{\uc(G)}$, for each SdMS $\X$ for $l \not \in \uc(G)$, we have $|\X|{=} \m-1$.
\end{restatable}

\section{Conclusion}

In this paper we defined and studied destructive smallest minimal supports for tournament solutions, as abductive explanations for why a given candidate loses an election. 
We showed that computing destructive SMSs is generally easier than computing their constructive version, in contrast with the results of Contet et al.~\cite{contet2026explaining} on the weighted uncovered set (with the open question remaining for Borda and Copeland). 
We also showed that destructive SMSs tend to be smaller than constructive ones. 
The size of the SdMSs is in $\mathcal{O}(\m^2)$ for the top cycle rule and in $\mathcal{O}(\n\m)$ for the remaining rules, in contrast
with the constructive case, where for the Borda or Copeland rule almost the whole tournament is necessary when all candidates are tied.

In order to find simple explanations, in this paper we put the emphasis on minimizing the size of minimal supports, measured as the total number of pairwise comparisons. 
We used size as a proxy for simplicity of the structure of dMSs, which in turn should result in a more accessible explanation.
However, depending on the specific application one might optimize for different metrics on dMSs. Further work, such as empirical studies, is needed to investigate which explanations are preferred by users.


\section*{Acknowledgments}

The authors thank the reviewers of ADT26 for their constructive comments and suggestions, which helped improve this paper.

This work is funded by the European Union. Views and opinions expressed are however those of the authors only and do not necessarily reflect those of the European Union or the European Research Council Executive Agency. Neither the European Union nor the granting authority can be held responsible for them. This work is supported by ERC grant 101166894 “Advancing Digital Democratic Innovation” (ADDI).

\begin{figure}
\centering
\includegraphics[height=1.5cm]{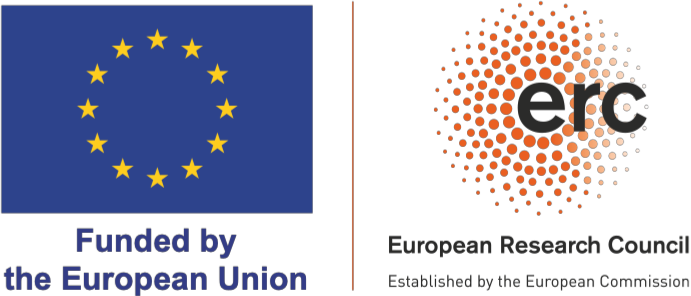}
\end{figure}

%
%
\bibliographystyle{splncs04}
\bibliography{biblio}

\newpage 

\appendix

\section*{Appendix}\label{sec:appendix}

\section{Proofs for Section~\ref{sec:pw}}

\subsection{Proofs for Section~\ref{sec:pwtc}}

\tcchara*

\begin{proof}
    Given a partial tournament $G =(\candidates,\mu)$, a candidate $l \in \candidates$.
    
    ($\implies$) We proceed by contrapose.
    Suppose for all $K\subseteq\candidates\setminus{\{l\}}$ with $K\neq \varnothing$, exists a pair $(c,c')\in K\times\candidates\setminus{K}$, $\mu(c,c')=0$.
    If $l\in\NL{\tc}{G}$, for each completion $\X$ of $G$, let $A_\X= \{c : c\in\candidates, l \text{ cannot reach } c \text{ in } \X\}$ be the set of candidates $l$ cannot reach in $\X$, $A_\X\neq \varnothing$.
    Let $\X_0=(\candidates,\mu_0)$ be such that $\X_0= \argmin_{\X\in [G]} |A_\X|$.
    Since $A_{\X_0}\subseteq\candidates\setminus{\{l\}}$ and $A_{\X_0}\neq \varnothing$, exists a pair $(c,c')\in A_{\X_0}\times\candidates\setminus{A_{\X_0}}$, $\mu(c,c')=0$.
    Let $\X_1=(\candidates,\mu_1)$ be a completion of $G$ such that $\mu_1(c,c')=0$, $\mu_1(c',c)=1$ and $\mu_1=\mu_0$ everywhere else.
    Clearly, for all $c\in \candidates\setminus{A_{\X_0}}$, $c\in \candidates\setminus{A_{\X_1}}$
    In particular, since $c'\in \candidates\setminus{A_{\X_0}}$, $c'\in \candidates\setminus{A_{\X_1}}$, i.e. $l$ can reach $c'$ in $\X_1$.
    Additionally, $\mu_1(c',c)=1$. Hence, $l$ can reach $c$, i.e., $c\in\candidates\setminus{A_{\X_1}}$.
    Hence $(\candidates\setminus{A_{\X_0}}) \subsetneq (\candidates\setminus{A_{\X_1}})$ or equivalently, $A_{\X_1} \subsetneq A_{\X_0}$.
    Contradiction. 
    Thus $l\not\in\NL{\tc}{G}$.
    
    ($\impliedby$) Suppose exists $K\subseteq\candidates\setminus{\{l\}}$ with $K\neq \varnothing$ such that for each pair $(c,c')\in K\times\candidates\setminus{K}$, $\mu(c,c')=1$ then clearly for each completion $X$ of $G$, $l$ cannot reach candidates in $K$ thus $l\not\in\tc(\X)$ and $l\in\NL{\tc}{G}$. \qed
\end{proof}

\subsection{Proofs for Section~\ref{sec:pwbo}}

\bochara*

\begin{proof}
    Schwartz studied this problem from the angle of network flow theory. The result was initially stated as follows in~\cite{schwartz1966possible}.
    
    Given a partial tournament, for any candidate $l \in \candidates$ and subset $K \subseteq \candidates\setminus{\{l\}}$, let $N(l,K)$ be the algebraic sum of half game behind by which $l$ trails each candidate of $K$, $P(l)$ be the number of games $l$ has still to play and $P(K)$ be the number of games the candidates of $K$ still have to play, omitting any in which they face one another.
    We have that for $l$ to be eliminated it is necessary and sufficient that exists a subset $K$ such that $$N(l,K)-kP(l)-P(K)> 0.$$

    We now show how to reach the new formulation from the original one.
    Given $\n$ voters and a partial $\n$-weighted tournament $G =(\candidates,\mu)$ with $|\candidates| = \m$, the half game behind score of a candidate $c$ is defined as $H(c) = \n(\m-1) + \sum_{c'\in \candidates} \mu(c,c') - \sum_{c'\in \candidates} \mu(c',c)$. Each of the $\n(\m-1)$ games that $c$ will have to play grants by default one half game, if $c$ wins the game then it grants a total of two half games (one half game plus one half game) and if $c$ loses that game grants no half game at all (one half game minus one half game). Then we simply have $N(l,K) = \sum_{c\in K} (H(c) - H(l))$.
    The number of game $c$ still has to play is naturally the total number of games $c$ has to play minus those $c$ already played, i.e., won or lost, hence $P(c) = \n(\m-1) - \sum_{c'\in \candidates} (\mu(c,c') +  \mu(c',c))$. It is the same idea for $P(K)$ but we only consider games against candidates not in $K$, $P(K) = \sum_{c\in K}\sum_{c'\not\in K} (\n - \mu(c,c') - \mu(c',c))$.

    \begin{align*}
        & N(l,K)-kP(l)-P(K)> 0 \\
        \iff & \sum_{c\in K} (H(c) - H(l)) - k\left(\n(\m-1) - \sum_{c\in \candidates} (\mu(l,c) +  \mu(c,l))\right) \\ & - \sum_{c\in K}\sum_{c'\not\in K} (\n - \mu(c,c') - \mu(c',c)) > 0\\
        \iff & \sum_{c\in K} \left(\sum_{c'\in \candidates} (\mu(c,c') - \mu(c',c)) - \sum_{c'\in \candidates} (\mu(l,c') - \mu(c',l))\right) \\ & - k\left(\n(\m-1) - \sum_{c\in \candidates} (\mu(l,c) +  \mu(c,l))\right) - \sum_{c\in K}\sum_{c'\not\in K} (\n - \mu(c,c') - \mu(c',c)) > 0\\
        \iff & \sum_{c\in K} \sum_{c'\in \candidates} (\mu(c,c') - \mu(c',c)) - k\sum_{c'\in \candidates} (\mu(l,c') - \mu(c',l)) - k\n(\m-1) \\ & - k\sum_{c\in \candidates} (\mu(l,c) +  \mu(c,l)) - \sum_{c\in K}\sum_{c'\not\in K} (\n - \mu(c,c') - \mu(c',c)) > 0\\
        \iff & \sum_{c\in K} \sum_{c'\in \candidates} (\mu(c,c') - \mu(c',c)) + 2k\sum_{c\in \candidates} \mu(c,l) - k\n(\m-1) \\ & - k(\m-k)\n + \sum_{c\in K}\sum_{c'\not\in K} (\mu(c,c') + \mu(c',c)) > 0\\
    \end{align*}

    In the first double sum, observe that if $c'\in K$, then both $c$ and $c'$ belong to $K$ and both $\mu(c,c') - \mu(c',c)$ and $\mu(c',c) - \mu(c,c')$ appear in the sum. 
    Additionally, remember that $\mu(c,c)=0$ so the case where $c=c'$ can be ignored. Hence, when $c'$ is in $K$, the pairs cancel in the sum and we have:

    \begin{align*}
        & \sum_{c\in K} \sum_{c'\in \candidates} (\mu(c,c') - \mu(c',c)) + 2k\sum_{c\in \candidates} \mu(c,l) - \n k(2\m-1-k) \\ & + \sum_{c\in K}\sum_{c'\not\in K} (\mu(c,c') + \mu(c',c)) > 0\\
        \iff & \sum_{c\in K} \sum_{c'\not\in K} (\mu(c,c') - \mu(c',c)) + 2k\sum_{c\in \candidates} \mu(c,l) - \n k(2\m-1-k) \\ & + \sum_{c\in K}\sum_{c'\not\in K} (\mu(c,c') + \mu(c',c)) > 0\\
        \iff & 2\sum_{c\in K} \sum_{c'\not\in K} \mu(c,c') + 2k\sum_{c\in \candidates} \mu(c,l) - \n k(2\m-1-k) > 0\\
        \iff & \sum_{c\in K}\sum_{c'\not\in K}\mu(c,c') + k\sum_{c\in\candidates}\mu(c,l) >\frac{1}{2}\n k(2\m-1-k)\\
        \iff & \sum_{c\in K}\sum_{c'\not\in K}\mu(c,c') + k\sum_{c\in\candidates}\mu(c,l) > k\n(\m-1) - \n \frac{k(k-1)}{2}
    \end{align*}

    Additionally, observe that since the maximum total number of pairwise comparisons a candidate can win in any completion $G'\in [G]$ is the total number of times it is compared less the number of pairwise comparisons it loses in $G$ we have:
    $$\max_{G'\in [G]} \sigma_{\borda}(l,G') = \n(\m-1) -\sum_{c\in\candidates}\mu(c,l).$$

    Similarly, the minimum total number of pairwise comparisons a candidate can win in any completion $G'\in [G]$ is the number of comparisons it wins in $G$. Moreover, independently of the outcome, a comparison between to candidate in a coalition always represent both a win and a loss. Only wins from members of the coalition over non-members increase the average Borda score. Hence we have:
    $$\min_{G'\in [G]} \mathop{{}\mathbb{E}}_{c\in K} \sigma_{\borda}(c,G') =\n \frac{k-1}{2} + \frac{1}{k}\sum_{c\in K}\sum_{c'\not\in K}\mu(c,c').$$

    Finally we have:

    $$N(l,K)-kP(l)-P(K)> 0
    \iff \min_{G'\in [G]} \mathop{{}\mathbb{E}}_{c\in K} \sigma_{\borda}(c,G') > \max_{G'\in [G]} \sigma_{\borda}(l,G').$$ \qed
\end{proof}

\subsection{Proofs for Section~\ref{sec:pwmmwuc}}

The proof of Theorem~\ref{th:wucchara} follows a similar approach to the one of Theorem~\ref{th:mmchara}. We show that a candidate is not weighted covered by any other one if and only if there exists a perfect matching in a specific bipartite graph. Note that this is not sufficient as a weighted covered candidate can still be in the uncovered set as long as it is not \emph{strictly} weighted covered. Typically, when two candidates weighted cover each other.

\wucchara*
\begin{proof}
    We proceed in a similar fashion to the proof of Theorem~\ref{th:mmchara}, we reduce the problem of determining if a candidate is weighted covered by a candidate in any completion to the problem of finding a perfect matching of a specific undirected unweighted bipartite graph. 

    Given $\n$ voters, a partial $\n$-weighted tournament $G =(\candidates,\mu)$ with $|\candidates| = \m$, and a candidate $l \in \candidates$.
    
    Given a completion $G'=(\candidates,\mu')$ of $G$, $l$ is a $\wuc$-winner if and only if it is not \emph{strictly} weighted covered that is that there is no $c_0\in \candidates$ such that $\mu'(c_0,l)\geq\mu'(l,c_0)$ and for all $c\in\candidates\setminus{\{l,c_0\}}$ $\mu'(c_0,c)\geq \mu'(l,c)$ and at least one of this inequalities is strict.

    We start by identifying the relaxed case where $l$ is not (simply) weighted covered in each completion $G'=(\candidates,\mu')$ of $G$, i.e., where for each candidate $c\in \candidates\setminus{\{l\}}$ we either have $\mu'(c,l) < \mu'(l,c)$ or a distinct candidate $c'\in \candidates\setminus{\{l,c\}}$ such that $\mu'(c,c') < \mu'(l,c')$. 

    Let the graph $H=(V_H,E_H)$ be such that $V_H=\candidates$ and $E_H=\{(c,c'):\forall (c,c')\in \left(\candidates\setminus{\{l\}}\right)^2,\,c\neq c' \land \mu(c,c') < \n - \mu(c',l)\}\cup\{(c,l):\forall c\in\candidates\setminus{\{l\}},\, \mu(c,l) < \left\lceil\frac{\n}{2}\right\rceil\}$ and $I_H^l=(X,Y,Z)$ be the $l$-deficient incidence graph of $H$.

    If $c$ is matched with $\{c,c'\}$ in $I_H^l$ then if $c'=l$ then $G$ is completed in $G'$ such that $\mu'(l,c) \geq \left\lceil\frac{\n+1}{2}\right\rceil$ which ensures that $l$ is strictly preferred to $c$ else when $c'\neq l$, $G$ is completed such that $\mu'(l,c') > \n - \mu'(c',c)$ which ensures that $l$ is strictly more preferred to $c'$ than $c$ is preferred to $c'$. In both cases, $l$ is not weighted covered by $c$.
    
    We now show that $l$ is not weighted covered by a candidate in any completion of $G$ if and only if there exists an $X$-perfect matching of in $I_H^l$.

    $(\implies)$ Suppose that $l$ is not weighted covered by a candidate in each completion of $G$, then there exists a completion $G'=(\candidates,\mu')$ of $G$ where $l$ is not weighted covered by any other candidate. 
    Thus, in $G'$, for all candidate $c$ distinct from $l$, 
    either $\mu'(l,c) > \left\lceil\frac{\n}{2}\right\rceil$ thus $\mu(c,l) \leq \mu'(c,l) = \n - \mu'(l,c) < \n - \left\lceil\frac{\n}{2}\right\rceil \leq \left\lceil\frac{\n}{2}\right\rceil$ and $c$ is matched with $\{c,l\}$ in $H$ or there exists $c'\in\candidates\setminus{\{l,c\}}$ such that $\mu'(l,c') > \mu'(c,c')$ which gives us $\n - \mu(c',l) \geq \n - \mu'(c',l) = \mu'(l,c') > \mu'(c,c') \geq \mu(c,c')$ and $c$ is matched with $\{c,c'\}$ in $H$. 
    Suppose $\{c,c'\}$ has been matched both to $c$ and $c'$, then we have that $\mu'(l,c') > \mu'(c,c')$ and $\mu'(l,c) > \mu'(c',c)$. 
    Hence, $\mu'(l,c') + \mu'(l,c) > \mu'(c,c') + \mu'(c',c) = \n$ and $\mu'(l,c)> \left\lceil\frac{\n}{2}\right\rceil$ or $\mu'(l,c')>\left\lceil\frac{\n}{2}\right\rceil$.
    Without loss of generality, suppose $\mu'(l,c)> \left\lceil\frac{\n}{2}\right\rceil$.
    Thus $\{c,\{c,l\}\}\in E_H$ which means that we can unmatch $c$ with $\{c,c'\}$ and match it with $\{c,l\}$. 
    This result in a $X$-perfect matching in $I_H^l$.
    
    $(\impliedby)$ Suppose there exists a $\X$-perfect matching in $I_H^l$, then for each candidate $c\in X$, exists $c'\in\candidates$ such that $c$ is matched with $\{c,c'\}$. 
    Now, let $G'=(\candidates,\mu')$ be a completion of $G$ where $l$ wins all its remaining comparisons, i.e., such that for all $c\in\candidates$, $\mu'(l,c)=\n -\mu(c,l)\geq \mu(l,c)$, and where $\mu'(c,c')=\mu(c,c')$ if $\{c,\{c,c'\}\}$ is part of the matching and where the rest of $G$ is completed arbitrarily.
    For each candidate $c$ distinct from $l$, if $\{c,\{c,l\}\}$ is part of the matching then $\{c,\{c,l\}\}\in E_H$ and $\mu(c,l)<\left\lceil\frac{\n}{2}\right\rceil$ else 
    there exists $c'\neq l$ such that $\{c,\{c,c'\}\}$ is part of the matching then $\{c,\{c,c'\}\}\in E_H$ and $\mu'(c,c')=\mu(c,c')<\mu(l,c')\leq \mu'(l,c)$, i.e., in both case $l$ is not weighted covered by $c$.
    Hence $l$ is not weighted covered by any other candidate in $G'$ which is a completion of $G$.

    Applying Lemma~\ref{lem:perfectmatchingcharadef}, we have that $l$ is a not weighted covered by a candidate in each completion $G$ if and only if $H$ has an acyclic connected component which does not contain $l$.
    Finally, by construction, $H$ has an acyclic connected component which does not contain $l$ if and only if there exists a tree $T=(K,E)$ with $K\subseteq\candidates\setminus{\{l\}}$ such that for each $c\in K$, $\mu(c,l)\geq\left\lceil\frac{\n}{2}\right\rceil$ and for each pair $(c,c')\in K\times\candidates\setminus{\{l\}}$ with $c\neq c'$, if $\{c,c'\}\not\in E$ then $\mu(c,c') + \mu(c',l)\geq\n$.

    Recall that $l\in\NL{\wuc}{G}$ when $l$ is strictly weighted covered in each completion of $G$.
    Since we know when $l$ is (simply) weighted covered in each completion of $G$, we now need to identify the sub-cases in which it is always strictly weighted covered.

    We show that it is necessary and sufficient to simply add the constraint that at least one of the inequalities in the tree of the previous characterization has to be strict.

    Suppose that no inequality is strict. 
    Then for each $c\in K$, $\mu(l,c)=\frac{\n}{2}$ and and for each $c'\in \candidates\setminus{\{l,c\}}$, if $\{c,c'\}\not\in E$ then $\mu(c,c') + \mu(c',l)=\n$.
    Let $G'=(\candidates,\mu')$ be a completion of $G$ where $l$ wins all its remaining comparisons, i.e., such that for all $c\in\candidates$, $\mu'(l,c)=\mu(l,c)=\frac{\n}{2}$, and where for each pair $(c,c')\in K\times\candidates\setminus{\{l\}}$ with $c\neq c'$, if $\{c,c'\}\in E$ then $\mu'(c,c')=\mu(c,c')=\frac{\n}{2}$ else $\mu(c,c') = \n - \mu(c',l) = \frac{\n}{2}$.
    It is clear that $l$ is not strictly weighted covered by any candidate in $K$.

    Now suppose that one of the inequalities in $T_0=(K_0,E_0)$ is strict. 
    Let us assume that it is associated to candidate $c_0 \in K$.
    If $T_0$ only contains one vertex, then $c_0$ strictly weighted covers $l$. 
    Else for $l$ not to be weighted covered by $c_0$, $G$ has to be extended in another partial tournament $G_1=(\candidates,\mu_1)$ where exists $c_1 \in K_0$, $c_1\neq c_0$ such that $\{c_0,c_1\}\in E_0$ where $\mu_1(c_0,c_1)<\mu_1(l,c_1)$. 
    Thus $\mu_1(c_1,c_0) \geq \n - \mu_1(c_0,c_1) > \n - \mu_1(l,c_1)$ and we have a new strict inequality $\mu_1(c_1,c_0) + \mu_1(l,c_1) > \n$.
    This means that exists $T_1=(K_1,E_1)$ a sub-tree of $T_0$ with $K_1\subseteq K_0\setminus{\{c_0\}}$ and $E_1\subsetneq E_0$ such that for each $c\in K_1$, $\mu_1(c,l)\geq\frac{\n}{2}$ and for each $c'\in \candidates\setminus{\{l,c_0,c\}}$, if $\{c,c'\}\not\in E_1$ then $\mu(c,c') + \mu(c',l)\geq\n$ where exists a strict inequality.
    This way we build a strictly decreasing sequence of trees with at least one strict inequality.
    Since we are working with a finite number of candidate, we end up with a tree restricted to one vertex which strictly weighted covers $l$ in each completion. \qed
\end{proof}

\section{Proofs for Section~\ref{sec:dmssize}}

\tcsize*

\begin{proof}
    Given a directed graph $G =(\candidates,\mu)$.
    A subset of vertices $K\in C$ is a \emph{strongly connected component} if: (\textit{i}) for each pair $(x,y)\in K^2$ with $x\neq y$, there exist a path from $x$ to $y$ and (\textit{ii}) $K$ is maximal, in the sense that no vertex can be added without violating condition (\textit{i}).

    Given a complete tournament $G =(\candidates,\mu)$ with $|\candidates| = \m$, and a losing candidate $l \in \candidates\setminus{\tc(G)}$,
    Let $\X=(\candidates,\mu_\X)$ be a dMS for $l \not \in \tc(G)$.
    
    Let $\sim$ be the equivalence relation on $\candidates$ defined by $x\sim y$ if and only if there are directed paths from $x$ to $y$ and $y$ to $x$.
    The equivalence classes of $\sim$ are the strongly connected components of $G$.
    We define the condensation graph $\widetilde{G}=(\widetilde{\candidates},\widetilde{\mu})$, i.e., the quotient graph of $G$ by $\sim$ with
    $\widetilde{\candidates}=\candidates/\sim$ and for $(x,y)\in{\widetilde{\candidates}}^2$ with $x\neq y$, $\widetilde{\mu}(x,y)=1$ if and only if exist $(u,v)\in x\times y$ such that $\mu(u,v)=1$.
    $\widetilde{G}$ is naturally a directed acyclic graph.
    However, since $G$ is complete, $\widetilde{G}$ is complete.
    Hence $\succ$, the order induced by $\widetilde{G}$ on the strongly connected components of $G$ is total.

    Let $p$ be the number of strongly connected components of $G$.
    Since $l \not \in \tc(G)$, $l$ cannot reach all candidates in $G$ and $p\geq 2$.
    Let $\{A_1,A_2,\dots,A_p\}$ be the set of strongly connected components of $G$ such that $A_1\succ A_2\succ \dots \succ A_p$.
    Let $2\leq i_l \leq p$ be such that $l\in A_{i_l}$.

    According to Theorem~\ref{th:tcchara}, exists $K\subseteq\candidates\setminus{\{l\}}$ with $K\neq \varnothing$ such that for each pair $(c,c')\in K\times\candidates\setminus{K}$, $\mu(c,c')=1$.
    But for each strongly connected component $A$, either $A\subseteq K$ or $A\cap K = \varnothing$ otherwise some candidates in $A$ could not reach other candidates in $A$.
    And for $1\leq i \leq p$, if $A_i\subseteq K$ then for all $1\leq j\leq i$, $A_j\subseteq K$ else there would exist $(A,A')\in\{A_1,A_2,\dots,A_p\}^2$ with $A\succ A'$ such that exists $(x,y)\in A\times A'$ such that $\mu(y,x)=1$.
    Hence, the only possibilities for $K$ are $\bigcup_{j=1}^i A_j$ with $1\leq i\leq i_l-1$.

    Additionally, to split $G$ between $K$ and its complementary $|K|(|\candidates|-|K|)$ pairwise comparisons are needed.
    Let the function $f:x\mapsto x(\m-x)$ be defined on $[|A_1|,|\bigcup_{j=1}^{i_l-1} A_j|]$. $f$ is clearly concave and thus reaches its minimum either in $x=|A_1|$ or $x=|\bigcup_{j=1}^{i_l-1} A_j|$.
    Since $A_1=\tc(G)$, $\bigcup_{j=i_l}^{p} A_j =\m - \bigcup_{j=1}^{i_l-1} A_j$ and $\bigcup_{j=i_l}^{p} A_j=\{c : c\in\candidates, l \text{ can reach } c \text{ in } G\}$,
    $|\X|= \displaystyle \min_{t\in \{\alpha, \beta\}} t(\m-t)$ where $\alpha = |\tc(G)|$ and $\beta = |\{c : c\in\candidates, l \text{ can reach } c \text{ in } G\}|$.

    If $|A_1|(\m-|A_1|) \leq |\bigcup_{j=1}^{i_l-1} A_j|(\m-|\bigcup_{j=1}^{i_l-1} A_j|)$ then
    let $\X=(\candidates,\mu_\X)$ be partial sub-tournament of $G$ such that for all $(c,c')\in A_1\times (\candidates\setminus{A_1})$, $\mu_\X(c,c')=1$ and $\mu_\X$ is null everywhere else.
    According to Theorem~\ref{th:tcchara}, $l\in \NL{\tc}{G}$.
    Since $|\X|=|A_1|(\m-|A_1|)$ matches the lower bound, $\X$ is a smallest dMS for $l\not\in\tc(G)$.
    The same reasoning holds when $|A_1|(\m-|A_1|) \geq |\bigcup_{j=1}^{i_l-1} A_j|(\m-|\bigcup_{j=1}^{i_l-1} A_j|)$.

    It naturally follows from the beginning of the proof, that an SdMS can be computed in polynomial time.
    Additionally, by concavity and symmetry, the worst case is reached when candidates are split in two equal parts, i.e., when $|TC(G)|=\left\lceil\frac{\m}{2}\right\rceil$ or $|TC(G)|=\left\lfloor\frac{\m}{2}\right\rfloor$ and the losing candidate is in $A_2$. In that case, $|\X| = \left\lceil\frac{\m}{2}\right\rceil\left\lfloor\frac{\m}{2}\right\rfloor$. \qed
\end{proof}

\bosize*

\begin{proof}
    We start by proving that $\left\lceil\frac{\n(\m-1)+1}{2}\right\rceil$ comparisons are enough when $l$ has below average performances, i.e., $\sigma_{\borda}(l) \leq \left\lfloor\frac{\n(\m-1)-1}{2}\right\rfloor$.

    Given $\n$ voters, a complete $\n$-weighted tournament $G =(\candidates,\mu)$ with $|\candidates| = \m$, a losing candidate $l \in \candidates\setminus{\borda(G)}$. Suppose that $\sigma_{\borda}(l,G') \leq \left\lfloor\frac{\n(\m-1)-1}{2}\right\rfloor$, then $\sum_{c\neq l} \mu_G(c,l) = \n(\m-1) - \sigma_{\borda}(l,G) \geq \left\lceil\frac{\n(\m-1)+1}{2}\right\rceil$. Let $\X \subseteq G$ such that $\sum_{c\neq l} \mu_\X(c,l) = \left\lceil\frac{\n(\m-1)+1}{2}\right\rceil$ and $\mu_\X$ is null otherwise. Clearly, $|\X| = \left\lceil\frac{\n(\m-1)+1}{2}\right\rceil$.
    Then for all complete tournaments $G'$ completing $\X$ we have $\sigma_{\borda}(l,G') = \n(\m-1) - \sum_{c\neq l} \mu_{G'}(c,l) \leq \n(\m-1) - \sum_{c\neq l} \mu_\X(c,l) = \n(\m-1) - \left\lceil\frac{\n(\m-1)+1}{2}\right\rceil = \left\lfloor\frac{\n(\m-1)-1}{2}\right\rfloor$. Moreover, $\sum_{c\neq l} \sigma_{\borda}(c,G') = \n\frac{\m(\m-1)}{2} - \sigma_{\borda}(l,G') = \n\frac{\m(\m-1)}{2} -\left\lfloor\frac{\n(\m-1)-1}{2}\right\rfloor = \n\frac{(\m-1)(\m-1)}{2} + \left\lceil\frac{1}{2}\right\rceil = \n\frac{(\m-1)(\m-1)}{2} +1$ and by the pigeonhole principle, among the $\m-1$ candidates distinct from $l$, there exists $c_0$ such that $\sigma_{\borda}(c_0,G') \geq \left\lceil\frac{\sum_{c\neq l} \sigma_{\borda}(c,G')}{\m-1}\right\rceil = \left\lceil\frac{\n(\m-1)}{2} + \frac{1}{\m-1} \right\rceil$. Hence, $\sigma_{\borda}(c_0,G') > \sigma_{\borda}(l,G')$ and $l$ is a necessary loser for $\X$ and $\X$ contains a dMS.

    Let us now prove that $\n(\m-1)+1 - \max_{c\in A} \mu(c,l)$ comparisons are always enough.

    Given $\n$ voters, a complete $\n$-weighted tournament $G =(\candidates,\mu)$ with $|\candidates| = \m$, a losing candidate $l \in \candidates\setminus{\borda(G)}$. Let $A=\{c:c\in\candidates, \sigma_{\borda}(c,G) > \sigma_{\borda}(l,G)\}$ and $c_0 = \argmax_{c \in A} \mu_G(c,l)$. Then, $\sum_{c\neq c_0} \mu_G(c_0,c) + \sum_{c\neq l} \mu_G(c,l) = \sigma_{\borda}(c_0,G) + \n(\m-1) - \sigma_{\borda}(l,G) \geq \sigma_{\borda}(l,G) + 1 + \n(\m-1) - \sigma_{\borda}(l,G) = \n(\m-1) + 1$.
    Let $\X \subseteq G$ be such that $\mu_\X(c_0,l) = \mu_G(c_0,l)$, $\sum_{c\neq l} \mu_\X(c_0,c) + \sum_{c\neq l} \mu_\X(c,l) = \n(\m-1) + 1$ and $\mu_\X$ is null otherwise. $|\X| = \n(\m-1) + 1 - \mu_\X(c_0,l) = \n(\m-1)+1 - \max_{c\in A} \mu(c,l)$.
    Then for all complete tournaments $G'$ completing $\X$ we have $\sum_{c\neq l} \mu_{G'}(c_0,c) + \sum_{c\neq l} \mu_{G'}(c,l) \geq \n(\m-1) + 1 \implies \sigma_{\borda}(c_0,G') + \n(\m-1) - \sigma_{\borda}(l,G') \geq \n(\m-1) + 1 \implies \sigma_{\borda}(c_0,G') > \sigma_{\borda}(l,G')$ and $l$ is a necessary loser for $\X$ and $\X$ contains a dMS.

    We now move on the lower bounds.

    Given $\n$ voters, a complete $\n$-weighted tournament $G =(\candidates,\mu)$ with $|\candidates| = \m$, a losing candidate $l \in \candidates\setminus{\borda(G)}$, let $\X=(\candidates,\mu_\X)$ be a dMS for $l \not \in \borda(G)$. 
    According to Theorem~\ref{th:bochara}, exists $K \subseteq \candidates\setminus{\{l\}}$ with $|K| = k$ such that
    
    $$\sum_{c\in K}\sum_{c'\not\in K}\mu(c,c') + k\sum_{c\in\candidates}\mu(c,l) > k\n(\m-1) - \n \frac{k(k-1)}{2}.$$

    We have that $|\X| \geq \displaystyle\sum_{c\in K}\sum_{c'\not\in K}\mu(c,c')$ and $|\X| \geq \displaystyle\sum_{c\in\candidates}\mu(c,l)$.
    Hence,
    \begin{align*}
        & |\X| + k|\X| > k\n(\m-1) - \n \frac{k(k-1)}{2}\\
        \iff & (k+1)|\X| > k\n(\m-1) - \n \frac{k(k-1)}{2}\\
        \iff & |\X| > \frac{1}{k+1}\left(k\n(\m-1) - \n \frac{k(k-1)}{2}\right)\\
        \iff & |\X| > \left(1-\frac{1}{k+1}\right)\left(\n(\m-1) - \n \frac{k-1}{2}\right)\\
        \iff & |\X| > \left(1-\frac{1}{k+1}\right)\left(\n\m - \n \frac{k+1}{2}\right)
    \end{align*}
    
    Let $f$ be the function $x \mapsto \left(1-\frac{1}{x+1}\right)\left(\n\m - \n \frac{x+1}{2}\right)$ defined on $[1;m-1] \to \mathbb{R}$. 
    $f$ is twice differentiable on its domain. 
    $f': x \mapsto \frac{1}{(x+1)^2}\left(\n\m - \n \frac{x+1}{2}\right)-\left(1-\frac{1}{x+1}\right)\frac{\n}{2}$ and 
    $f'': x \mapsto \frac{-2}{(x+1)^3}\left(\n\m - \n \frac{x+1}{2}\right)-\frac{1}{(x+1)^2}\frac{\n}{2}-\frac{1}{(x+1)^2}\frac{\n}{2}=\frac{-2\n\m}{(x+1)^3}$.
    $f''$ is clearly negative on the interval $[1;m-1]$, thus $f$ is concave and minimal at $x=1$ or $x=\m-1$.

    $f(1)=\left(1-\frac{1}{1+1}\right)\left(\n\m - \n \frac{1+1}{2}\right)=\frac{\n(\m-1)}{2}$.

    $f(\m-1)=\left(1-\frac{1}{\m-1+1}\right)\left(\n\m - \n \frac{\m-1+1}{2}\right)=\frac{\n(\m-1)}{2}$.

    Hence, $|\X| > \frac{\n(\m-1)}{2}$.
    Since $|\X|$ is an integer we have $|\X| \geq \left\lceil\frac{\n(\m-1)+1}{2}\right\rceil$.

    Suppose now that $\sigma_{\borda}(l) > \left\lfloor\frac{\n(\m-1)-1}{2}\right\rfloor$. 
    We have that $\sum_{c\in\candidates}\mu_\X(c,l)\leq\sum_{c\in\candidates}\mu(c,l)=\n(\m-1)-\sigma_{\borda}(l)<\n(\m-1) - \left\lfloor\frac{\n(\m-1)-1}{2}\right\rfloor = \left\lceil\frac{\n(\m-1)+1}{2}\right\rceil$.

    Thus, $\sum_{c\in\candidates}\mu_\X(c,l)\leq \left\lceil\frac{\n(\m-1)-1}{2}\right\rceil$, it is not possible to take $K=\candidates\setminus{\{l\}}$ anymore and $|\X| > \left\lceil\frac{\n(\m-1)+1}{2}\right\rceil$.

    If we look at $$\sum_{c\in K}\sum_{c'\not\in K}\mu(c,c') +  k\sum_{c\in\candidates}\mu(c,l)> k\n(\m-1) - \n \frac{k(k-1)}{2}$$
    it is clear that:
    \begin{itemize}
        \item for $c\in K$ $\mu(c,l)$ contributes $k+1$ times in the left hand
        \item for $c,c'\in K\times(\candidates\setminus{\{l\}})$ $\mu(c,c')$ contributes $1$
        \item for $c\not\in K$ $\mu(c,l)$ contributes $k$
        \item for $c,c'\in (\candidates\setminus{K})\times(\candidates\setminus{\{l\}})$ $\mu(c,c')$ contributes $0$.
    \end{itemize}

    Hence, maximizing $\sum_{c\in\candidates}\mu(c,l)$ is never damaging to obtain smallest dMSs and we can take it to its maximal value, $\n(\m-1) - \sigma_{\borda}(l,G)$.

    We have:
    \begin{align*}
        & \sum_{c\in K}\sum_{c'\not\in K}\mu(c,c') +  k\sum_{c\in\candidates}\mu(c,l)> k\n(\m-1) - \n \frac{k(k-1)}{2}\\
        \implies & \sum_{c\in K}\sum_{c'\in \candidates\setminus{(K\cup\{l\})}}\mu(c,c') + \sum_{c\in K}\mu(c,l) + k\sum_{c\in\candidates}\mu(c,l)> k\n(\m-1) - \n \frac{k(k-1)}{2}\\
        \implies & \sum_{c\in K}\sum_{c'\in \candidates\setminus{(K\cup\{l\})}}\mu(c,c') + \sum_{c\in K}\mu(c,l) + k(\n(\m-1) - \sigma_{\borda}(l,G))> k\n(\m-1) - \n \frac{k(k-1)}{2}\\
        \implies & \sum_{c\in K}\sum_{c'\in \candidates\setminus{(K\cup\{l\})}}\mu(c,c')> k\sigma_{\borda}(l,G) - \n \frac{k(k-1)}{2} - \sum_{c\in K}\mu(c,l).
    \end{align*}

    With a MILP solver, we can fin $K^*$ and $\mu^*$ which minimize $\sum_{c\in K}\sum_{c'\in \candidates\setminus{(K\cup\{l\})}}\mu(c,c')$ while satisfying this constraint.

    Finally, we can take $\X=(\candidates,\mu_\X)$ with $\forall c\in K,c'\in \candidates\setminus{(K\cup\{l\}})$, $\mu_\X(c,c')=\mu^*(c,c')$, $\mu_\X(c,l)=\mu(c,l)$ and $\mu_\X$ null elsewhere.
    Hence, $|\X|=\sum_{c\in K^*}\sum_{c'\in \candidates\setminus{(K^*\cup\{l\})}}\mu^*(c,c')+\n(\m-1)-\sigma_{\borda}(l,G)$.
    \qed
\end{proof}


\cosize*


\mmsize*

\begin{proof}
    In the case $\m=2$, ensuring that a strict majority of voters prefers the other candidate to $l$ is clearly necessary and sufficient so we move to the case where $\m\geq 3$.
    
    We start by proving that $(\sigma_{\mm}(l)+1)(\m-3) + \n + 1$ and $(\sigma_{\mm}(l)+1)(\m-2) + \n + 1$ are enough depending on the cases.

    Given $\n$ voters, a complete $\n$-weighted tournament $G =(\candidates,\mu)$ with $|\candidates| = \m$, a losing candidate $l \in \candidates\setminus{\mm(G)}$. Since $l\not\in \mm(G)$, exists $c_0\in \candidates\setminus{\{l\}}$ such that $\sigma_{\mm}(c_0) > \sigma_{\mm}(l,G)$. Additionally, let $c_1 = \argmax_{c\neq l} \mu(c,l)$, by definition $\mu(c_1,l) = \n - \sigma_{\mm}(l,G)$. Let $\X=(\candidates,\mu_\X)$ be a such that for all $c\in \candidates\setminus{\{c_0\}}$, $\mu_\X(c_0,c) = \sigma_{\mm}(l,G) +1$, $\mu_\X(c_1,l) = \n - \sigma_{\mm}(l,G)$ and the rest of $\mu_\X$ is null. Then, for all completion $G'$ of $\X$, $\sigma_{\mm}(c_0,G') = \min_{c\neq c_0} \mu_{G'}(c_0,c) \geq \min_{c\neq c_0} \mu_{\X}(c_0,c) = \sigma_{\mm}(l,G) +1$ and $\sigma_{\mm}(l,G') = \min_{c\neq l} \mu_{G'}(l,c) = \n - \max_{c\neq l} \mu_{G'}(c,l) \geq  \n - \max_{c\neq l} \mu_{\X}(c,l) \geq \n - \mu_\X(c_1,l) = \n - (\n - \sigma_{\mm}(l,G)) = \sigma_{\mm}(l,G)$. Hence, $\sigma_{\mm}(c_0,G') > \sigma_{\mm}(l,G')$ and $l$ is a necessary loser for $\X$ and $\X$ contains a dMS.
    
    If $c_0\neq c_1$, $|X| = (\m-1)(\sigma_{\mm}(l,G) +1) + \n - \sigma_{\mm}(l,G) = (\sigma_{\mm}(l)+1)(\m-2) + \n + 1$ and if $c_0=c_1$, having $\mu_\X(c_0,l)= \n - \sigma_{\mm}(l,G)$ ensures that $\mu_\X(c_0,l)=\sigma_{\mm}(l,G) + 1 $ ($\sigma_{\mm}(l,G) \leq \left\lfloor \frac{n-1}{2}\right\rfloor$ otherwise, $l$ is a Condorcet winner and thus a Maximin winner) so we can save $\sigma_{\mm}(l,G) + 1$ comparisons and $|X| = (\m-2)(\sigma_{\mm}(l,G) +1) + \n - \sigma_{\mm}(l,G) = (\sigma_{\mm}(l)+1)(\m-3) + \n + 1$.

    We now prove the matching lower bounds.
    Given $\n$ voters, a complete $\n$-weighted tournament $G =(\candidates,\mu)$ with $|\candidates| = \m$, a losing candidate $l \in \candidates\setminus{\mm(G)}$, according to Theorem~\ref{th:mmchara}, for $l$ to be the necessary $\mm$-loser of a partial sub-tournament $G'\subseteq G$, $G'$ needs to contain two elements:
    \begin{itemize}
        \item a pairwise comparison in which $l$ is beaten to limit its maximal possible maximin score,
        \item a set of candidates arranged in a tree where each candidate in the tree beats the non adjacent candidates in the tree with a strictly stronger weight than $l$'s possible maximin score.
    \end{itemize}

    Observe that if the tree contains $k\geq 1$ candidates, $k(\m-3)+2$ edges needs to have a stronger weight than $l$'s possible maximin score.
    Each time the maximal possible maximin score of $l$ decreases by one, we can save $k(\m-3)+2\geq 1$ comparisons. 
    Hence, to obtain the smallest dMS, we want the maximal possible maximin score of $l$ to be as low as possible, i.e., fixed to $\sigma_{\mm}(l,G)$.
    Let $c_1 = \argmax_{c\neq l} \mu(c,l)$, we can fix $\mu_\X(c_1,l)=\n-\sigma_{\mm}(l,G)$ to bound the maximal possible maximin score of $l$ to $\sigma_{\mm}(l,G)$ using $\n-\sigma_{\mm}(l,G)$ comparisons.
    
    Additionally, note that $k(\m-3)+2$ increases with $k$ so we want as few candidates in tree as possible. 
    Since $l\not\in\mm(G)$, exists a candidate $c_0\neq l$ such that $\sigma_{\mm}(c_0) > \sigma_{\mm}(l)$. 
    Thus, we can take the tree reduced to the single vertex $c_0$ as $c_0$ beats any other candidate with strictly more than $\sigma_{\mm}(l)$ comparisons.
    This gives us a lower bound of $(\m-1)(\sigma_{\mm}(l)+1)$ pairwise comparisons for this constraint.

    If for all $c\in \candidates$ such that $\sigma_{\mm}(c) > \sigma_{\mm}(l)$,  $\mu(c,l)< \n - \sigma_{\mm}(l)$ both constraints cannot be combined and both sets of pairwise comparisons are disjoint.
    Hence, a lower bound of $\n-\sigma_{\mm}(l,G) + (\m-1)(\sigma_{\mm}(l)+1) = (\sigma_{\mm}(l)+1)(\m-2) + \n + 1$  corresponding to the third case of the theorem.
    In the other case, we can have $c_0=c_1$ and by fixing $\mu_\X(c_1,l)=\n-\sigma_{\mm}(l,G)$, we have $\mu_\X(c_1,l)\geq \sigma_{\mm}(l,G) +1$ without requiring additional comparisons. 
    This allows to save an extra $\sigma_{\mm}(l,G) +1$ which gives a final lower bound of $(\sigma_{\mm}(l)+1)(\m-3) + \n + 1$.

    It naturally follows from the beginning of the proof, that an SdMS can be computed in polynomial time.
    Finally, since the score of a losing candidate cannot be greater than $\left\lfloor \frac{n-1}{2}\right\rfloor$ otherwise, it is a Condorcet winner and thus a Maximin winner, we have that for each SdMS $\X$ for maximin, $(\left\lfloor \frac{n-1}{2}\right\rfloor+1)(\m-2) + \n + 1 = \left\lceil \frac{n+1}{2}\right\rceil(\m-2) + \n + 1$. \qed
\end{proof}

\wucsize*
\begin{proof}
    We first provide upper bounds on the number of pairwise comparisons required in each case by exhibiting partial sub-tournaments where $l$ is a necessary loser in each case.
    
    Given $\n$ voters, a complete $\n$-weighted tournament $G =(\candidates,\mu)$ with $|\candidates| = \m$, a losing candidate $l \in \candidates\setminus{\wuc(G)}$.

    We start with the second case and we show that there always exists a dMS of size $\n(\m-2) + \left\lceil\frac{\n+1}{2}\right\rceil$.
    Since $l$ is a losing candidate in $G$, exists $c_0\in\candidates\setminus{\{l\}}$ such that $c_0$ weighted covers $l$, i.e., for all candidates $c\in\candidates\setminus{\{c_0\}}$, $\mu(c_0,c)\geq\mu(l,c)$.
    Hence, for all candidates $c\in\candidates\setminus{\{c_0,l\}}$, $\mu(c_0,c)\geq \n - \mu(c,l)$ or equivalently, $\mu(c_0,c) + \mu(c,l)\geq \n$.
    Let $\X\subseteq G$ be a partial sub-tournament of $G$ with $\X=(\candidates,\mu_\X)$ and $\mu_\X(c_0,l)=\left\lceil\frac{\n}{2}\right\rceil$, for all candidates $c\in\candidates\setminus{\{c_0,l\}}$, $\mu_X(c_0,c) + \mu_\X(c,l)\geq \n$.
    For each completion $\Y=(\candidates,\mu_\Y)$ of $\X$, $\mu_\Y(c_0,l)\geq\frac{\n}{2}\geq\mu_\Y(l,c_0)$.
    Additionally, for all candidates $c\in\candidates\setminus{\{c_0,l\}}$, $\mu_Y(c_0,c) + \mu_\Y(c,l) \geq \mu_X(c_0,c) + \mu_\X(c,l)\geq \n$.
    Thus, for all candidates $c\in\candidates\setminus{\{c_0,l\}}$, $\mu_Y(c_0,c) \geq \n - \mu_\Y(c,l) = \mu_\Y(l,c)$.
    We have that $l$ is weighted covered by $c_0$ in $\Y$ and $l\in\NL{\wuc}{\X}$.
    Finally, for at least one of the inequalities to be strict we need an additional comparison if $\n$ is even and none if $\n$ is odd as $\mu_\Y(c_0,l)\geq\frac{\n}{2}$ is already strict. 
    We encode this by changing the $\left\lceil\frac{\n}{2}\right\rceil$ into $\left\lceil\frac{\n+1}{2}\right\rceil$
    Since, $|\X|=\n(\m-2) + \left\lceil\frac{\n+1}{2}\right\rceil$, $\n(\m-2) + \left\lceil\frac{\n+1}{2}\right\rceil$ pairwise comparisons are enough.

    We now move on to the first case.
    Suppose $p=\sum_{c\in\candidates\setminus{\{l,c_0,c_1\}}}\mu(c,l)-\n(\m-3)-\left\lfloor\frac{\n}{2}\right\rfloor>0$ where $(c_0,c_1)\in(\candidates\setminus{\{l\}})^2$ maximizes $\sum_{c\in\candidates\setminus{\{l,c_0,c_1\}}}\mu(c,l)$ with $c_0\neq c_1$, $\mu(c_i,l)\geq \frac{\n}{2}$, for all $c\in\candidates\setminus{\{l,c_0,c_1\}}$, $\mu(c_i,c)+\mu(c,l)\geq\n$ for $i\in\{0,1\}$ and at least one of the inequalities is strict.
    Let $\X\subseteq G$ be a partial sub-tournament of $G$ with $\X=(\candidates,\mu_\X)$ and $\mu_\X(c_i,l)=\left\lceil\frac{\n}{2}\right\rceil$, for all $c\in\candidates\setminus{\{l,c_0,c_1\}}$, $\mu_\X(c,l)=\mu(c,l)$ and $\mu_\X(c_i,c)=\n-\mu(c,l)$ for $i\in\{0,1\}$.
    For each completion $\Y=(\candidates,\mu_\Y)$ of $\X$, since $\mu_\Y(c_0,c_1) + \mu_\Y(c_1,c_0)=\n$, $\mu_\Y(c_0,c_1) \geq \left\lceil\frac{\n}{2}\right\rceil$ or $\mu_\Y(c_1,c_0) \geq \left\lceil\frac{\n}{2}\right\rceil$.
    Without loss of generality, suppose $\mu_\Y(c_0,c_1) \geq \left\lceil\frac{\n}{2}\right\rceil$.
    Since $\mu_\X(c_1,l)=\left\lceil\frac{\n}{2}\right\rceil$, $\mu_\Y(l,c_1)=\n-\mu_\Y(c_1,l)\leq \n-\mu_\X(c_1,l) =\n - \left\lceil\frac{\n}{2}\right\rceil =\left\lfloor\frac{\n}{2}\right\rfloor$.
    Hence $\mu_\Y(c_0,c_1) \geq \mu_\Y(l,c_1)$.
    Additionally, $\mu_\Y(c_0,l)\geq \mu_\X(c_0,l)\geq \left\lceil\frac{\n}{2}\right\rceil$ and $\mu_\Y(l,c_0)=\n- \mu_\Y(c_0,l)\leq \n-\left\lceil\frac{\n}{2}\right\rceil \leq \left\lfloor\frac{\n}{2}\right\rfloor$.
    Thus $\mu_\Y(c_0,l) \geq \mu_\Y(l,c_0)$.
    Moreover, for all $c\in\candidates\setminus{\{l,c_0,c_1\}}$, $\mu_\Y(c_0,c)+\mu_\Y(c,l)\geq\mu_\X(c_0,c)+\mu_\X(c,l)=\mu(c,l) + \n - \mu(c,l) =\n$. 
    Thus $\mu_\Y(c_0,c)\geq \n - \mu_\Y(c,l)=\mu_\Y(l,c)$.
    Finally,
    \begin{align*}
        & 2\left\lceil\frac{\n}{2}\right\rceil+\sum_{c\in\candidates\setminus{\{l,c_0,c_1\}}}\mu(c,l) + 2\sum_{c\in\candidates\setminus{\{l,c_0,c_1\}}}(\n-\mu(c,l))\\
        & = 2\left\lceil\frac{\n}{2}\right\rceil+2\sum_{c\in\candidates\setminus{\{l,c_0,c_1\}}}\n - \sum_{c\in\candidates\setminus{\{l,c_0,c_1\}}}(\mu(c,l))\\
        & = 2\left\lceil\frac{\n}{2}\right\rceil+ 2\n(\m-3) - \sum_{c\in\candidates\setminus{\{l,c_0,c_1\}}}(\mu(c,l))\\
        & = \left\lceil\frac{\n}{2}\right\rceil+ \n(\m-3) - p
    \end{align*}

    Additionally, like in the previous case, for at least one of the inequalities to be strict we need an additional comparison if $\n$ is even and none if $\n$ is odd as $\mu_\Y(c_0,l)\geq\frac{\n}{2}$ is already strict. 
    We encode this by changing the $\left\lceil\frac{\n}{2}\right\rceil$ into $\left\lceil\frac{\n+1}{2}\right\rceil$.
    
    \begin{align*}
        |\X| & = \left\lceil\frac{\n+1}{2}\right\rceil+ \n(\m-3) - p < \left\lceil\frac{\n+1}{2}\right\rceil+ \n(\m-3)
    \end{align*}
    Since, $|\X|=\n(\m-3) + \left\lceil\frac{\n+1}{2}\right\rceil-p$, $\n(\m-3) + \left\lceil\frac{\n+1}{2}\right\rceil-p$ pairwise comparisons are enough.

    We now prove the matching lower bounds.
    
    Since dMSs are partial tournaments where $l$ is a necessary loser, we use the characterization provided in Theorem~\ref{th:wucchara} to obtain lower bounds on the riquered number of pairwise comparisons.
    Given $\n$ voters, a complete $\n$-weighted tournament $G =(\candidates,\mu)$ with $|\candidates| = \m$, a losing candidate $l \in \candidates\setminus{\wuc(G)}$.
    According to Theorem~\ref{th:wucchara}, for each dMS $\X=(\candidates,\mu_\X)$ for $l \not \in \wuc(G)$, we have that exists a tree $T=(K,E)$ with $K\subseteq\candidates\setminus{\{l\}}$ such that for each $c\in K$, $\mu_\X(c,l)\geq\frac{\n}{2}$ and $\forall c'\in \candidates\setminus{\{l\}}$, $\forall c \in K$, if if $\{c,c'\}\not\in E$ then $\mu_\X(c,c') + \mu_\X(c',l)\geq\n$ and at least one of the inequalities is strict.
    In particular, $\forall c'\in \candidates\setminus{\{l\}}$, unless $\forall c \in K$, $\{c,c'\}\in E$, i.e, $T$ is a star graph with center $c'$, $\exists c \in K$ such that $\mu_\X(c,c') + \mu_\X(c',l)\geq\n$.

    Note that if $|K|=1$ and we denote $K=\{c\}$, then $T$ is a star graph with center $c$.
    If $|K|=2$ and $K=\{c_0,c_1\}$, then $T$ is both a star graph with center $c_0$ and a star graph with center $c_1$.
    If $|K| \geq 3$, there exists at most on candidate $c$ such that $T$ is a star graph with center $c$ otherwise $T$ would have a cycle.

    Additionally, if $T$ is a star graph with center $c$, we have $c\in K$ and $\mu_\X(c,l)\geq\frac{\n}{2}$.

    Hence, if $|K|=1$ or $|K| \geq 3$, if $T$ is a star graph we have $\m-2$ independent inequalities of the form $\mu_\X(c,c') + \mu_\X(c',l)\geq\n$ and one of the form $\mu_\X(c,l)\geq\frac{\n}{2}$, and at least one has to be strict thus $|\X|\geq (\m-2)\n+\left\lceil\frac{\n+1}{2}\right\rceil$ else we have $\m-1$ independent inequalities of the form $\mu_\X(c,c') + \mu_\X(c',l)\geq\n$ and $|\X|\geq (\m-1)\n\geq(\m-2)\n+\left\lceil\frac{\n+1}{2}\right\rceil$.

    If $|K|=2$ and $K=\{c_0,c_1\}$, our set of constrains is $\mu_\X(c_i,l)\geq \frac{\n}{2}$ and for all $c\in\candidates\setminus{\{l,c_0,c_1\}}$, $\mu_\X(c_i,c)+\mu_\X(c,l)\geq\n$ for $i\in\{0,1\}$ and at least one inequality has to be strict.
    Since $\mu_\X(c,l)$ appears both in $\mu_\X(c_0,c)+\mu_\X(c,l)\geq\n$ and $\mu_\X(c_1,c)+\mu_\X(c,l)\geq\n$, to obtain the smallest dMS, it is clear that we have to take its highest possible value which is $\mu(c,l)$.
    In that case, if we ignore the strict inequality constraint,
    \begin{align*}
        |\X| & = 2\left\lceil\frac{\n}{2}\right\rceil+\sum_{c\in\candidates\setminus{\{l,c_0,c_1\}}}\mu(c,l) + 2\sum_{c\in\candidates\setminus{\{l,c_0,c_1\}}}(\n-\mu(c,l))\\
        & = 2\left\lceil\frac{\n}{2}\right\rceil+ 2\n(\m-3) - \sum_{c\in\candidates\setminus{\{l,c_0,c_1\}}}(\mu(c,l)).
    \end{align*}

    Once again, for at least one of the inequalities to be strict we need an additional comparison we change one $\left\lceil\frac{\n}{2}\right\rceil$ into $\left\lceil\frac{\n+1}{2}\right\rceil$.
    
    If $p=\sum_{c\in\candidates\setminus{\{l,c_0,c_1\}}}\mu(c,l) - \n(\m-3)-\left\lfloor\frac{\n}{2}\right\rfloor>0$ as in the first case of the theorem then 
    \begin{align*}
        |\X| & = \left\lceil\frac{\n}{2}\right\rceil+\left\lceil\frac{\n+1}{2}\right\rceil+ 2\n(\m-3) - \sum_{c\in\candidates\setminus{\{l,c_0,c_1\}}}(\mu(c,l))\\
        & = \left\lceil\frac{\n+1}{2}\right\rceil+ \n(\m-3) - p\\
        & < \left\lceil\frac{\n+1}{2}\right\rceil+ \n(\m-2).
    \end{align*}
    Else taking $|K|=2$ does not help to obtain smaller dMSs and we go back to the general approach with $|K|=1$.

    It naturally follows from the beginning of the proof, that an SdMS can be computed in polynomial time.
    Finally, since the score of a losing candidate cannot be greater than $\left\lfloor \frac{n-1}{2}\right\rfloor$ otherwise, it is a Condorcet winner and thus a Maximin winner, we have that for each SdMS $\X$ for the weighted uncovered set, $|\X|\leq \n(\m-2)+\left\lceil\frac{\n+1}{2}\right\rceil$. \qed
\end{proof}


\ucsize*

\begin{proof}
    In the unweighted case ($\n=1$),
    $\sum_{c\in\candidates\setminus{\{l,c_0,c_1\}}}\mu(c,l) \leq \sum_{c\in\candidates\setminus{\{l,c_0,c_1\}}} 1 \leq \m-3$.
    Hence, $p=\sum_{c\in\candidates\setminus{\{l,c_0,c_1\}}}\mu(c,l) - (\m-3)-\left\lfloor\frac{1}{2}\right\rfloor\leq 0$ and the first case of Theorem~\ref{th:wucsize} never occurs.
    \qed
\end{proof}

\end{document}